%% file: main.tex
\documentclass[11pt, a4paper]{article}

\usepackage{amsmath}            
\usepackage{mathtools}          
\usepackage{amsthm}             
\usepackage{amsfonts}           
\usepackage{amssymb}            

\newtheorem{theorem}{Theorem}[section]
\newtheorem{lemma}{Lemma}[section]

\usepackage[utf8]{inputenc}     
\usepackage[T1]{fontenc}        
\usepackage{newtxtext}          
\usepackage{newtxmath}          
\usepackage{microtype}          

\usepackage[margin=1in]{geometry}

\usepackage[english]{babel}
\usepackage{graphicx}
\usepackage{subcaption}
\usepackage[colorlinks=true, allcolors=black]{hyperref}
\usepackage{bm}
\usepackage{placeins}
\usepackage{tcolorbox}

\usepackage{tikz}
\usetikzlibrary{shapes, arrows.meta, positioning, calc}

\title{Spectral Higher-Order Neural Networks}
\author{
    Gianluca Peri\textsuperscript{1}, 
    Timoteo Carletti\textsuperscript{2}, 
    Duccio Fanelli\textsuperscript{1}, 
    Diego Febbe\textsuperscript{1} \\[2ex]
    \small \textsuperscript{1}University of Florence, Italy \\
    \small \textsuperscript{2}University of Namur, Belgium \\[1ex]
    \small \texttt{\{gianluca.peri, duccio.fanelli, diego.febbe\}@unifi.it} \\
    \small \texttt{timoteo.carletti@unamur.be}
}

\date{}

\begin{document}
\maketitle

\begin{abstract}
\noindent
Neural networks are fundamental tools of modern machine learning. The standard paradigm assumes binary interactions (across feedforward linear passes) between inter-tangled units, organized in sequential layers. Generalized architectures have been also designed that move beyond pairwise interactions, so as to account for higher-order couplings among computing neurons. Higher-order networks are however usually deployed as augmented graph neural networks (\textsc{gnn}s), and, as such, prove solely advantageous in contexts where the input exhibits an explicit hypergraph structure. Here, we present \textit{Spectral Higher-Order Neural Networks} (\textsc{shonn}s), a new algorithmic strategy to incorporate higher-order interactions in general-purpose, feedforward, network structures. \textsc{shonn}s leverages a reformulation of the model in terms of spectral attributes. This allows to mitigate the common stability and parameter scaling problems that come along weighted, higher-order, forward propagations.
\end{abstract}

\section{Introduction}
 Higher-Order Neural Networks (\textsc{honn}s) can claim a long history in the field of machine learning, despite having being overlooked for many years. To the best of our knowledge, the first conceptual attempt to account for higher order schemes dates back to 1986, with the introduction of the so called \textit{sigma-pi} units~\cite{Shin1991ThePN}. These are structures that first group the inputs into clusters, then compute the product within each cluster, and finally sum the obtained products. \textit{Sigma-pi} units have been conceived, from the very beginning, as neural networks. The acronym \textsc{honn} became popular only one year after their actual formalization~\cite{Rumelhart1986ParallelDP, Giles1987LearningIA}. The field of modern deep learning has however steered to progressively favor dense neural network models: pairwise interactions among neurons are combined into weighted sums before being fed, as an input, to apposite non linear functions. This orientation has persisted despite the fact that architectures featuring higher-order interactions potentially offer greater expressivity than those restricted to binary couplings. This is likely attributable to the ensuing parameters scaling and associated computational costs. For a typical feedforward neural network model, also termed \textit{Multi Layer Perceptron} (\textsc{mlp}), one can expect a parameter scaling of the order $O(N^2)$, for any given pair of nested layers of size $N$. Conversely, for a higher-order network with triadic interactions (include two body correlations from the departing layer) the total parameter count scales as $O(N^3)$. The latter approach is impractical due to the prohibitively high training times required for real-world applications. Furthermore, the necessity of navigating an extensive parameter space often results in suboptimal solutions. Despite various efforts to mitigate the parameter scaling problem,  \textsc{honn}s have generally failed to achieve the same level of mainstream adoption as multi-layer perceptron \textsc{mlp} models. Two notable examples include the aforementioned \textit{pi-sigma} units~\cite{Shin1991ThePN}, and $\Pi$-nets \cite{chrysos2021deep, chrysos2022augmenting}. 

In recent years, several machine learning subfields have emerged from the integration of various higher-order processing techniques. Prominent among these is the attention mechanism employed in modern transformer architectures; other notable examples include Gated Linear Units (\textsc{glu}) and factorization machines for recommendation systems \cite{Vaswani2017AttentionIA, Shazeer2020GLUVI, Rendle2010FactorizationM}. The field's most direct adoption of higher-order neural networks arguably occurs within Hypergraph Neural Networks (\textsc{hgnn}s) and Simplicial Neural Networks (\textsc{snn}s) \cite{Feng2018HypergraphNN, Ebli2020SimplicialNN}. Both architectures generalize Graph Neural Networks (\textsc{gnn}s) to inputs with higher-order structures by adapting the message-passing algorithms to hypergraphs. It should be noted, however, that these models not only suffer from the curse of dimensionality ($O(N^3)$ parameter scaling, a lower bound which applies to the most conservative setting where just triadic interactions are incorporated) but are also inherently constrained to datasets possessing an explicit hypergraph structure.

The objective of our work is to address the intrinsic limitations that have hindered the broader adoption and advancement of higher-order neural networks. 

This is achieved through \textit{Spectral Higher-Order Neural Networks} (\textsc{shonn}s), where triadic interactions are integrated into the standard \textsc{mlp} forward pass. The computational cost (as measured by the parameters scaling) is reduced to $O(N^2)$  due to an effective weight reparametrization that exploits the spectral attributes of the involved transfer matrices
\cite{giambagli2021machine}. Moreover, we demonstrate that \textsc{shonn}s possess full expressivity, as they are capable of universal approximation for any continuous function. In a manner similar to \textsc{hgnn}s and \textsc{snn}s, our proposed scheme operates on complete hypergraphs of neurons; consequently, higher-order interactions are distributed throughout the architecture rather than being confined to isolated functional blocks. 

In the following, we present the mathematical foundation of \textsc{honn}s, focusing on the specific case of triadic interactions. These serve as the fundamental building blocks of the higher-order neural structure. We will then turn to demonstrate their universal approximation properties, as anticipated above. Dedicated benchmark tests have been performed, both in classification and regression modalities. Our results suggest that higher-order networks, integrated with the spectral parametrization, achieve superior performance over traditional low-order counterparts, without incurring additional computational overhead.

\section{Results}
\subsection{Higher-Order Layers}\label{sec:higher_order_layers}

\begin{figure}[t]
     \centering
     \begin{subfigure}[b]{0.3\textwidth}
         \centering
         \input{perceptron1}
         \caption{Standard connectivity of a neural network.}
         \label{fig:left}
     \end{subfigure}
     \hfill 
     \begin{subfigure}[b]{0.3\textwidth}
         \centering
         \input{perceptron2}
         \caption{Standard higher-order (\textit{triadic}) forward propagation structure.}
         \label{fig:middle}
     \end{subfigure}
     \hfill
     \begin{subfigure}[b]{0.3\textwidth}
         \centering
         \input{perceptron3}
         \caption{Forward propagation of a spectral higher-order neural network.}
         \label{fig:right}
     \end{subfigure}
     
     \caption{Cartoon representing different neural architectures. The standard neural networks (panel \ref{fig:left}) with input neurons, $x_i$, generating  output signals, $y_k$, via weighted averages, i.e., linear combinations (colored arrows) followed by the application of a local nonlinearity (not shown). The standard triadic higher-order network (panel \ref{fig:middle}) is obtained by adding to the previous architecture, weighted sums of hyperlinks, mimicking the two body interaction $x_{i_1}x_{i_2}$ (symbolized by the curved arrows connecting a couple $(x_{i_1},x_{i_2})$ to an output $y_k$, each pair with its own specific color). The spectral higher-order networks (panel \ref{fig:right}), reduces the number of used parameters by exploiting the spectral decomposition. This yields an effective parameter sharing among hyperlinks (the curved arrows connecting a couple $(x_{i_1},x_{i_2})$ to several outputs $y_{k}$, share the same color).}
     \label{fig:forward-passes}
\end{figure}
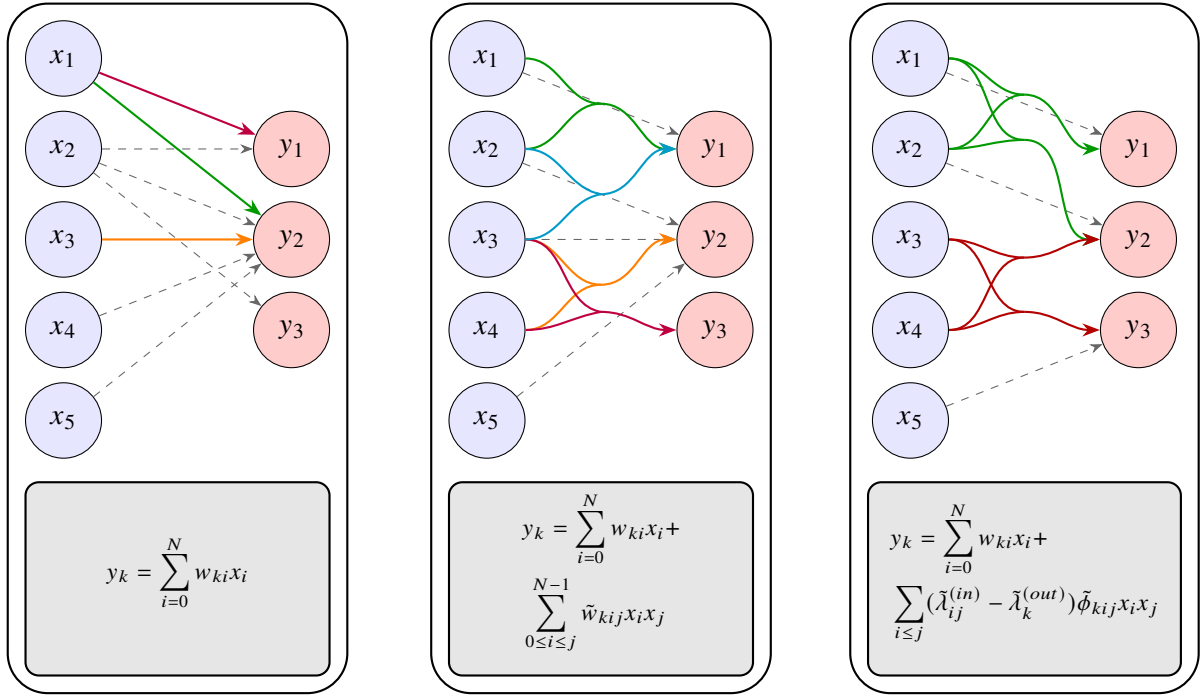

Consider a standard \textsc{mlp} and focus on two adjacent layers. Assume $x_i$ to denote the signal associated to the departing layer, made of $N$ individual units. The  signal $y_k$ referred to the $k$-th neuron at the arriving layer can be computed as follows the usual forward propagation, namely:
\begin{equation}
    y_k = \sum _{i=0}^N w_{ki}x_i,
    \label{eq:1-simplex}
\end{equation}
where $w_{ki}$ is the $k,i$-th element of the rectangular weight matrix $W_{K \times N+1}^{(1)}$ and $K$ stands for the size of the second layer (see panel (a) Fig.~\ref{fig:forward-passes}). Notice that the sum in~\eqref{eq:1-simplex} runs on $N+1$ elements, because we have chosen to deploy the scalar bias via a dedicated \textit{bias neuron}, $x_N$. An element-wise non-linearity $\sigma : \mathbb{R}^K \to \mathbb{R}^K$ is then applied to the obtained activation vector $\bm{y}=(y_1,\dots,y_K)^\top$, to expand the model's hypothesis space.

We can extend the above scheme beyond $1$-simplexes by accounting for higher-order interactions between layers. We begin by adding a quadratic term and thus recast ~\eqref{eq:1-simplex} in the form: 
\begin{equation}
    y_k = \sum _{i=0}^N w_{ki}x_i + \sum _{0 \leq i \leq j}^{N-1} \tilde{w}_{kij}x_ix_j\, ,
    \label{eq:2-simplex}
\end{equation}
where $\tilde{w}_{kij}$ is now the element of a tridimensional tensor $K \times N \times N$. See middle panel of Fig.~\ref{fig:forward-passes} for a graphical representation of the ensuing pattern of interaction. Without loss of generality, we have chosen not to include a bias term in the higher-order interaction defined in \eqref{eq:2-simplex}. Henceforth, we shall assume this distinction without further mention. Notably, the non-linearity is integrated directly into the network topology, thereby obviating the requirement for an explicit non linear post-processing activation step - a sensible departure from standard architectures.

Model~\eqref{eq:2-simplex} can be seen as a second order approximation of the full  \textsc{honn} expansion~\cite{Giles1987LearningIA}
\begin{equation}
    y_k = w_{k}^{(0)} + \sum_{i} w_{ki}^{(1)} x_i + \sum_{i,j} w_{kij}^{(2)} x_i x_j + \sum_{i,j,\ell} w_{kij\ell}^{(3)} x_i x_j x_\ell + \dots \, .
    \label{eq:historical}
\end{equation}
However, it should be noted that in practical implementations, the aforementioned sum must be truncated; under these conditions, the expressive power of~\eqref{eq:historical} remains to be proven. Additionally, the rapid growth in the parameters' count renders any practical implementation computationally prohibitive. From a more general perspective, model~\eqref{eq:historical} can be interpreted as the Taylor expansion~\footnote{By setting $n=2$ one can approximate $f(x_i, x_j; \bm{\theta})$ as follows
\begin{equation}
    f(x_i, x_j; \bm{\theta}) = f(\bm{0}; \bm{\theta})
    + \frac{\partial f}{\partial x_i} x_i
    + \frac{\partial f}{\partial x_j} x_j
    + \frac{1}{2} \left[  \frac{\partial^2 f}{\partial x_i^2}x_i^2 
    + 2  \frac{\partial^2 f}{\partial x_i \partial x_j} x_i x_j
    +  \frac{\partial^2 f}{\partial x_j^2}x_j^2 \right] + ... \,.
\end{equation}
The $0$-th and $1$-th orders can be identified with the bias and the linear term, i.e., a standard \textsc{mlp}. The truly novel term, is the second order one, $x_ix_j$, $i,j\in\{1,2\}$.} of the following non linear update rule
\begin{equation}
    y_k = \sum _{i=0}^N w_{ki}x_i + \sum _{0 \leq i \leq j \leq ... \leq n}^{N-1} f\big(x_i,x_j,...,x_n; \bm{\theta}(k,i,j,...,n)\big)\, ,
    \label{eq:general}
\end{equation}
where $f$ is a generic coupling function and  $\bm{\theta}(k,i,j,...,n)$  sets the pattern of active higher order correlations. As we shall prove, model choice~\eqref{eq:2-simplex} possesses universal approximation capabilities, provided the transformation is iteratively applied through a deep sequence of hidden layers. Stated differently, one can approximate any continuous function by coupling sufficiently many quadratic interactions of the type accommodated for in~\eqref{eq:2-simplex} (see Appendix \ref{sec:direct_space_universal_theorem}). This implies a drastic decrease in complexity compared to~\eqref{eq:historical}, thus enabling more scalable implementations.

A primary factor hindering the widespread adoption of the latter models within the deep learning community has to do with a phenomenon closely related to the curse of dimensionality. The number of parameters needed for an usual \textsc{mlp} network of fixed depth grows as $N^2$, where $N$ refers to size of the layers. The triadic hyperedge model suffers from an $O(N^3)$ parameter explosion, which substantially increases the computational burden in practical applications.

The main result of this work is a solution to this latter problem, via an efficient reparameterization scheme that induces weight sharing across the architecture. The recipe represents a generalization of the \textit{spectral parametrization} which has been recently proposed as a viable alternative to the usual optimization frameworks~\cite{chicchi2021} (see section \ref{sec:the_spectral_parametrization}). Applied to the classical \textsc{mlp}~\eqref{eq:1-simplex}, the spectral formulation yields
\begin{equation}
    y_k = \sum _{i} (\lambda_i^{(in)}-\lambda _k^{(out)})\phi_{ki}x_i\, ,
    \label{eq:spectral-param}
\end{equation}
where elements $\phi_{ki}$ define the eigenvectors of a square transfer operator,  $\bm{\lambda}^{(in)} \in \mathbb{R}^{N}$ and $\bm{\lambda}^{(out)} \in \mathbb{R}^K$ the associated eigenvalues, referred to departure and arrival nodes, respectively. By optimizing only the eigenvalues and preserving the original eigenvectors the cost is $O(N)$, we achieve thus a significant efficiency gain over the original setting, which requires direct weight optimization. By extending the spectral paradigm to accommodate higher-order forward propagation through triadic interactions, as specified in~\eqref{eq:2-simplex}, one eventually obtains: 
\begin{equation}
    y_k = \sum _{i} (\lambda_i^{(in)}-\lambda _k^{(out)})\phi_{ki}x_i + \sum _{i\leq j} (\tilde{\lambda} _{ij}^{(in)}-\tilde{\lambda} _k^{(out)})\tilde{\phi}_{kij}x_ix_j\, .
    \label{eq:extended-spectral-param}
\end{equation}

The derivation of the above result is given in the Methods section. A pictorial representation of the spectral higher order algorithm is provided in the right panel of Fig.~\ref{fig:forward-passes}. Hold constant  the scalar quantities $\tilde\phi_{kij}$ defined within the triadic summation, and just train the generalized eigenvalues $\tilde \lambda_{ij}^{\text{(in)}}$, $\tilde \lambda_k^{\text{(out)}}$ together with the entire parameter suite associated with the linear forward pass (both the eigenvectors and the eigenvalues). Consequently, the parametric complexity required to optimize the generalized model—including two-body interactions—is reduced to 
$O(N^2)$, in contrast to the $O(N^3)$ scaling inherent in a naive implementation. We refer to Appendix~\ref{sec:parameters_scaling} for a comprehensive analysis of the involved parameter scalings. In summary, by adopting the spectral ansatz, the triadic higher-order model can be trained with the same parametric complexity as a standard \textsc{mlp} \footnote{Recall that, under the family of \textsc{gd} optimizers routinely used for deep learning application, every omitted parameter represents a corresponding reduction in the total count of partial derivatives required during backpropagation.}. In the following section, we will begin to explore the practical side of dealing with a  spectral triadic architecture. As a first pedagogical example we shall focus on a perceptron, a simple two layers module, run against benchmark datasets for classification tasks. 

\subsection{Triadic Perceptrons}

\begin{figure}[t]
    \centering
    
    \begin{subfigure}{\linewidth}
        \begin{tcolorbox}[colback=white, colframe=black!50, arc=4mm, auto outer arc, boxrule=1pt]
            \centering
            \includegraphics[width=0.98\linewidth]{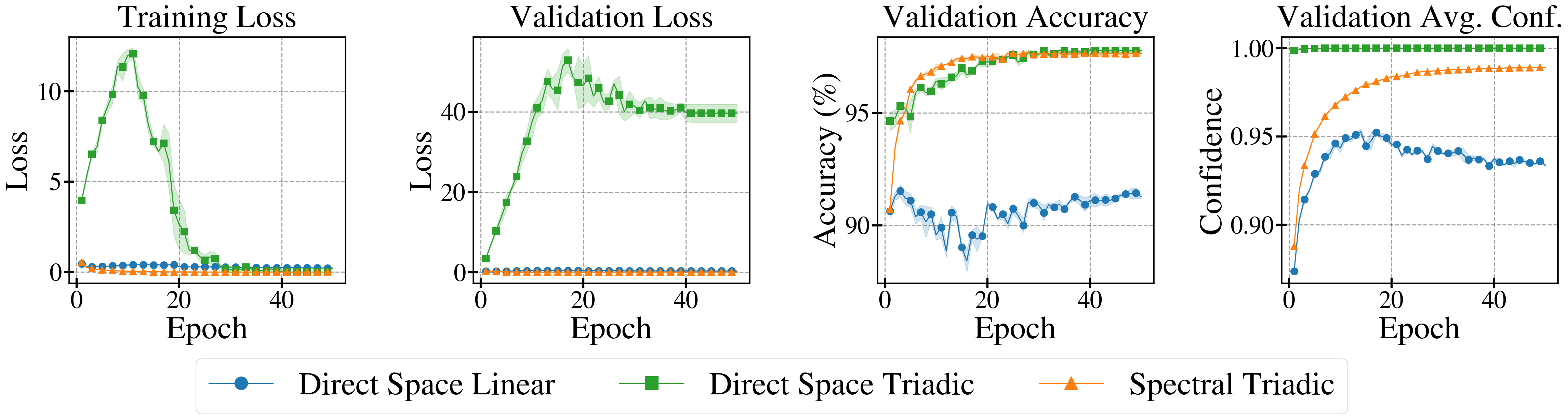}
            \caption{Results on \textsc{mnist}}
            \label{fig:mnist}
        \end{tcolorbox}
    \end{subfigure}
    
    \vspace{1em}
    
    \begin{subfigure}{\linewidth}
        \begin{tcolorbox}[colback=white, colframe=black!50, arc=4mm, auto outer arc, boxrule=1pt]
            \centering
            \includegraphics[width=0.98\linewidth]{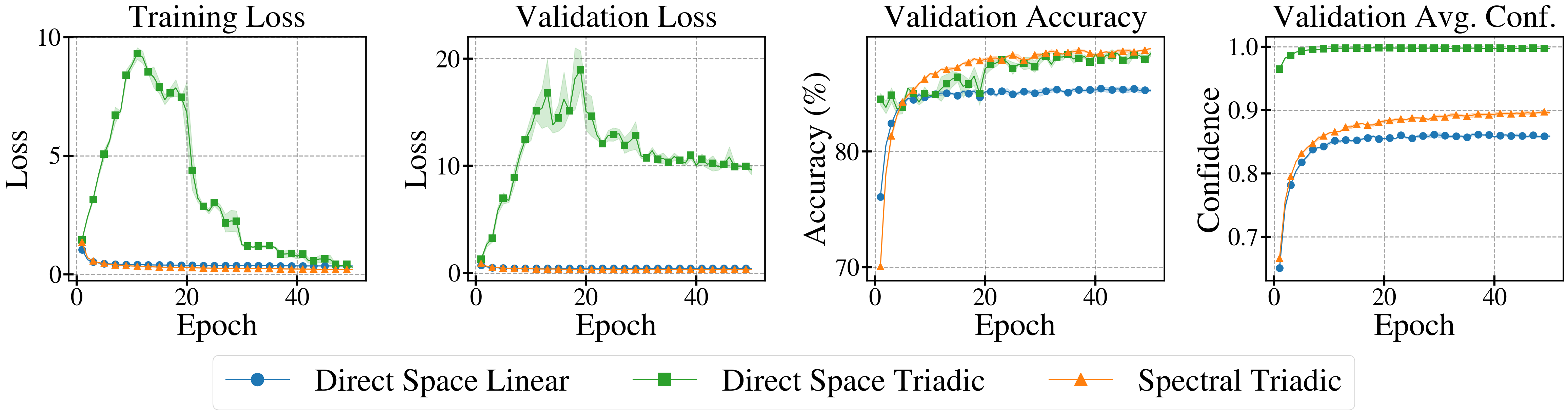}
            \caption{Results on \textsc{fashion-mnist}}
            \label{fig:fashion_mnist}
        \end{tcolorbox}
    \end{subfigure}
    
    \vspace{1em}
    
    \caption{Results of the perceptrons' training on \textsc{mnist} and \textsc{fashion-mnist} (via \textit{Adam} optimizer). We performed \textit{learning rate warm-up}, followed by a \textit{reduce-learning-rate-on-plateau} protocol, to guard against a possible dependence of the results on an unlucky hyperparameter choice. The direct space triadic model is plagued by instability and confidence saturation problems, while the standard perceptron lacks in expressivity. At variance, the spectral triadic model (i) achieves the same performances of the standard triadic network, with a substantially more efficient parameter scaling and (ii) it is also way more stable.}
    \label{fig:combined_mnist}
\end{figure}

To assess the fundamental performance of the \textsc{shonn} framework, we first conduct tests on a single layer perceptron, operated in classification mode against the \textsc{mnist} dataset \cite{Rosenblatt1958ThePA, lecun1998mnist}. The results are reported in Figure \ref{fig:mnist}, and show clear performance and stability gaps, favoring the spectral triadic model. Note that for these, and subsequent, experiments we trained $3$ separate models for each architecture, to acquire an uncertainty measure. To ensure the robustness of our findings we evaluated the perceptron models on \textsc{fashion-mnist} \cite{Xiao2017FashionMNISTAN}. The results, displayed in Figure \ref{fig:fashion_mnist}, corroborate our previous conclusions.

Even from these preliminary tests, it is clear that the spectral paradigm not only substantially improves parameter scaling but also mitigates the numerical instability and gradient saturation issues inherent in standard higher-order neural networks trained in direct space.  In the next section we will turn to considering the multi-layered generalization of the two layered models.

\subsection{Triadic \textsc{mlp}s}
\label{sec:mlps}

\begin{figure}[t]
    \centering
    \includegraphics[width=0.8\textwidth]{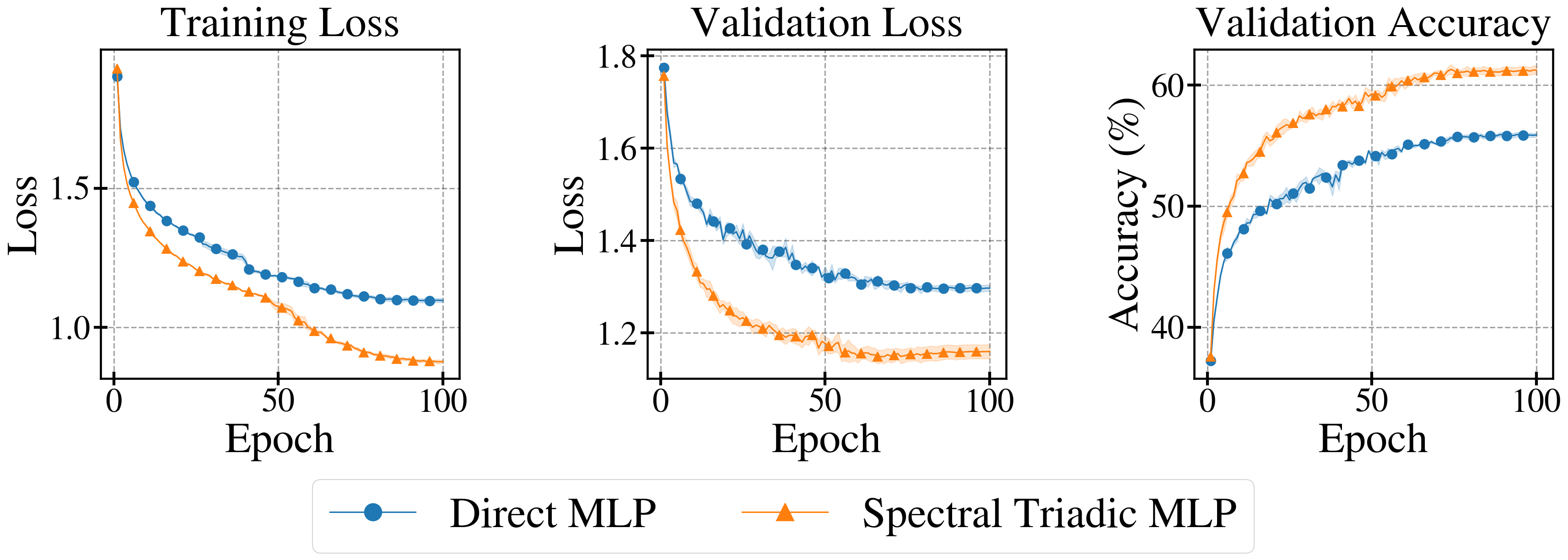}
    \caption{Standard \textsc{mlp} vs. spectral triadic \textsc{mlp} on CIFAR-10. The models were trained with the \textit{Adam} optimizer, following a \textit{halving-lr-on-plateau} scheduler. The results show a clear advantage for the triadic architecture.}
    \label{fig:cifar10-mlp}
\end{figure}

\begin{figure}[t]
    \centering
    \includegraphics[width=0.8\textwidth]{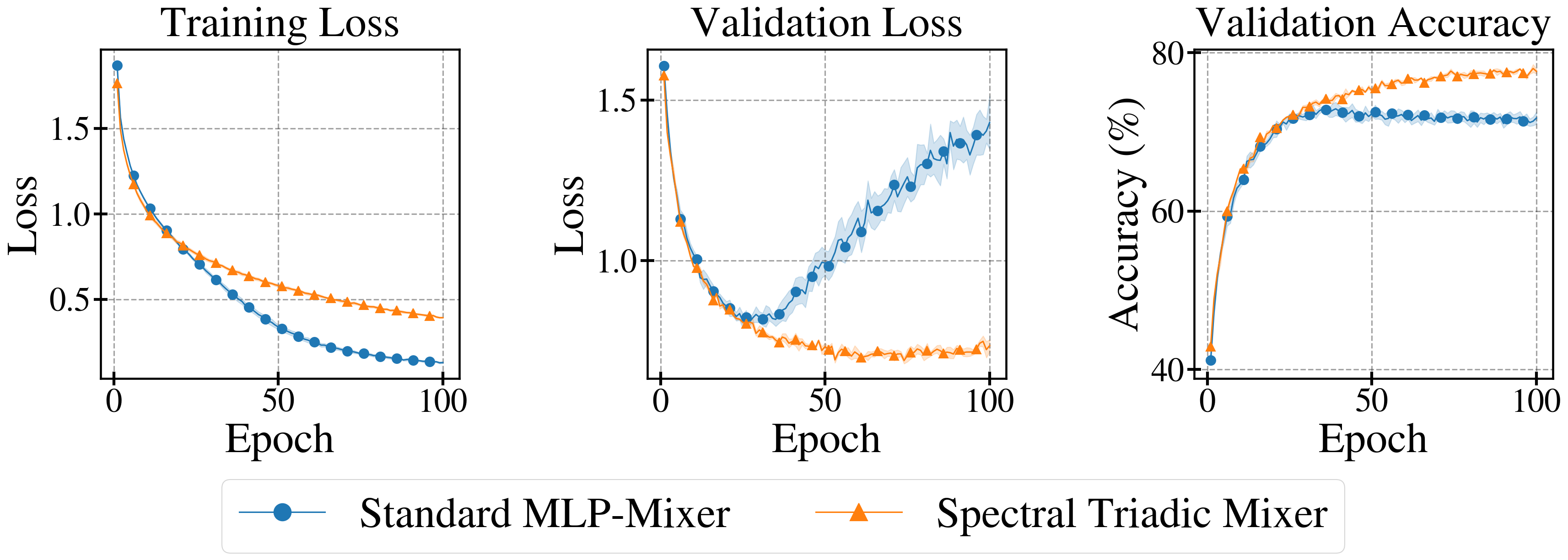}
    \caption{Standard \textsc{mlp-mixer} vs a spectral triadic version of it, on CIFAR-10. For this experiment the \textit{Adam} optimizer was used with fixed learning rate. From the results, it seems that the spectral forward propagation not only shows a sufficient degree of expressivity, but also acts as an implicit regularizer, preventing overfitting.}
    \label{fig:cifar10-mixer}
\end{figure}

Perceptrons can be staked recursively to create \textit{Multi Layer Perceptrons} with increased expressivity power. The same applies to triadic perceptrons of the type analyzed above. A standard $3$-layer neural network is the simplest \textsc{mlp} architecture
\begin{equation}
    y_k = \sum _i^H w_{ki}^{(2)}\sigma\left(\sum _j w_{ij}^{(1)} x_j\right)\, ,
    \label{eq:standardmlp}
\end{equation}
where $w_{ij}^{(1)},w_{ki}^{(2)}$ are the elements of the weight matrices of the first and second layer respectively, and $\sigma$ is an appropriate non-linear filter (usually a \textit{ReLU} function) to potentiate the degree of expressivity beyond the trivial linear setting. It can be proven that~\eqref{eq:standardmlp} approximate any continuous function with arbitrary precision as long as the hidden dimension $H$ is large enough, given the complexity of the task~\cite{LESHNO1993861}. This high degree of expressiveness, however, entails an interpretability trade-off; it is notoriously difficult to elucidate the inner workings of this class of models.

From~\eqref{eq:2-simplex}, it is evident that an elementary triadic perceptron can represent quadratic mappings, offering a clear advantage in expressivity over standard perceptron models. Furthermore, by hierarchically stacking triadic perceptrons a multi-layer configuration reveals a notable progression in the model's hypothesis space, e.g., a $2$-layer triadic \textsc{mlp} is able to represent polynomial mappings up to the fourth degree. In this sense, it proves equivalent to a higher-order perceptron with pentadic interactions (5-body interactions, a further step up in the higher-order hierarchy), but with a significantly lowered parameter count. It can be shown that a $N$-layer triadic \textsc{mlp} is able to represent any polynomial with degree $2^{N-1}$ (see Methods section for more details). 
Building upon this result, we demonstrate that any continuous function can be represented by a sufficiently wide and deep triadic \textsc{mlp}. This holds for both the direct-space parametrization in Eq.~\eqref{eq:2-simplex} (see Appendix~\ref{sec:direct_space_universal_theorem}) and the spectral parametrization, as discussed in Section \ref{sec:expressiveness_TMLP}. Recall that this latter setting comes with a concomitant compression of the relevant parameter space.
We also remark that the incremental expressivity, as gained by progressively adding more triadic layers, could enhance the interpretability of the proposed framework.

To assess empirically the capabilities of spectral higher-order \textsc{mlp}s, we tested their performances on CIFAR-10 \cite{Krizhevsky2009LearningML}. Specifically, we compared the run on a standard, $4$ layers, \textsc{mlp} architecture, equipped with a non-linear activation function (ReLU), against an equivalently shaped spectral triadic \textsc{mlp}. The results are shown in Figure~\ref{fig:cifar10-mlp}, and remarkably indicate better performances for our novel architecture, despite the absence of a point-wise non topological activation function $\sigma$.

Having established the reliability of the triadic \textsc{mlp}, it is important to highlight its potential for integration into existing architectures. One may in fact consider replacing conventional \textsc{mlp} modules with triadic layers to effectively boost model capabilities. To substantiate this claim, we considered a state-of-the-art architecture for vision, the \textsc{mlp-mixer}~\cite{Tolstikhin2021MLPMixerAA}, which we modified by replacing \textsc{mlp} modules with triadic spectral perceptrons. Even in such vastly different context, spectral higher-order networks exhibit a significant performance advantage, as clearly depicted in Figure \ref{fig:cifar10-mixer}. We refer the reader to Appendix \ref{sec:regression} for extended regression analysis, which confirms the above conclusions.

\section{Discussion}
\label{sec:discussion}
We here introduced the spectral parameterization for higher-order, triadic, forward propagation. This new framework opens up the perspective to incorporate hyperlinks into feedforward neural networks, without sacrificing the customary $O(N^2)$ scaling. 
Indeed, to the best of our knowledge, the proposed spectral parametrization exhibits the most favorable parameter scaling among all existing higher-order neural network architectures. Our experiments demonstrate that the newly proposed  hyper-linked neural architectures exhibit high efficacy even in the absence of explicit, pointwise activation functions, while simultaneously mitigating the stability issues typically associated with triadic configurations trained in direct space. The introduced class of networks also possess a highly configurable hypothesis space: indeed, the maximum complexity of the expressed mapping can be tailored by adjusting the network depth. This tunability can be extremely useful in context where the  polynomial order of the targeted mapping is approximately known. For this reason, we identify the application of these networks to physical contexts (\textit{e.g.} via \textsc{pinn}s \cite{Raissi2019PhysicsinformedNN}) as a key future research direction.  Furthermore, the results referred to deep architectures (\textsc{mlp}s, \textsc{mlp-mixer}s) already indicate that higher-order spectral networks  outperform their standard, direct space, counterparts. At last, we have derived universal approximation theorems for this class of triadic architectures, rigorously granting their expressive power. The collective bulk of theoretical and empirical results here presented clearly shows that elements of hypergraph theory can be fruitfully integrated into deep learning frameworks to enhance model performance. Future research is required to comprehensively evaluate the full potential of this interdisciplinary framework.

\section{Methods}
\subsection{The Spectral Parametrization} \label{sec:the_spectral_parametrization}
Assume a weight matrix $W_{K \times N}$ of the links between neural layers of size $N$ and $K$ respectively. The spectral parametrization approach \cite{giambagli2021machine} focuses on the implied adjacency matrix between the $N+K$ neurons of the underlying bipartite graph:
\begin{equation}
A = \begin{bmatrix}\mathbb{O}_{N \times N} & \mathbb{O} _{N \times K}\\W_{K \times N} & \mathbb{O}_{K \times K}\end{bmatrix}
\end{equation}
The $A$ matrix operates the forward pass on the global activation vector of the bipartite graph, usually called $\bm{a}$. Following the action of $A$ on $\bm{a}$, the signal that interests the neurons of the last layer is eventually stocked in the final $K$ entries of $\bm{a}$. The forward pass of information mediated by $A$ is independent of its diagonal elements. This allows us to define an augmented version of $A$ with non-zero diagonal entries, which serves as the fundamental operator for implementing the spectral parametrization. Specifically we can express the modified adjacency matrix in the form $A=\Phi \Lambda \Phi ^{-1}$, with $\Lambda \in \mathbb{R}^{(N+K) \times (N+K)}$ diagonal eigenvalues matrix, and with $\Phi$:
\begin{equation}
    \Phi = \begin{bmatrix}\mathbb{I}_{N \times N} & \mathbb{O} _{N \times K}\\\phi_{K \times N} & \mathbb{I}_{K \times K}\end{bmatrix}
\end{equation}
where $\mathbb{I}$ denotes the identity matrix. At last, given the properties of $\Phi^{-1}$, is possible to show \cite{giambagli2021machine} that the elements $w_{ki}$ which populate the sub-diagonal block of $A$ can be rewritten as:
\begin{equation}
    w_{ki} = (\lambda _i^{(in)} - \lambda _k^{(out)}) \phi_{ki}
\end{equation}
Where $\lambda _i^{(in)}$ is the $i$-th element of the first portion of the eigenvalue diagonal, and $\lambda _k^{(out)}$ is the $k$-th of the last one.

This novel parametrization for the network's weights offers the versatility to achieve a multitude of different goals, spanning  pruning techniques, input feature relevance detection, architecture search strategies, and, as we shall see, parameters' reduction paradigms \cite{chicchi2021, buffoni2022spectral, chicchi2024automatic, Peri2025}.

\subsection{Spectral Parametrization of Triadic Networks}

In equation \eqref{eq:2-simplex} the $\tilde{w}_{kij}$ tensor element can be thought as a matrix element, $\tilde{w}_{ki'}$: here $i'$ represents an appropriate index which enumerates the set of unordered pairs $(i,j), \ 0 \leq i \leq j<N$ in lexicographical order:
\[
(0,0) \to 0, \ (0,1) \to 1, \ ..., \ (0,N-1) \to N-1, \ (1,1) \to N, \ ..., \ (N-1,N-1) \to \frac{N(N+1)}{2} - 1
\]
In close form $i'$ is given by:
\begin{equation}
i' = j-i+\sum _{k=0}^{i-1}(N-k).
\label{eq:lexicographical_map}
\end{equation}
The lexicographical bijection \eqref{eq:lexicographical_map} allows the deployment of the spectral parametrization for triadic interactions: we first use \eqref{eq:lexicographical_map} to map the tridimensional weight tensor $[\tilde{w}_{kij}]$ into the weight matrix $[\tilde{w}_{ki'}]$; we can then simply express $[\tilde{w}_{ki'}]$ via the spectral parametrization. From \eqref{eq:2-simplex} we get:
\begin{equation}
    y_k = \sum _{i} (\lambda_i^{(in)}-\lambda _k^{(out)})\phi_{ki}x_i + \sum _{i'} (\tilde{\lambda} _{i'}^{(in)}-\tilde{\lambda} _k^{(out)})\tilde{\phi}_{ki'}x_{i'}
\end{equation}
with $x_{i'} = x_ix_j$. Finally we can expand $i'$ back to $2$ dimensions:
\begin{equation}
    y_k = \sum _{i} (\lambda_i^{(in)}-\lambda _k^{(out)})\phi_{ki}x_i + \sum _{0 \leq i \leq j} (\tilde{\lambda} _{ij}^{(in)}-\tilde{\lambda} _k^{(out)})\tilde{\phi}_{kij}x_ix_j,
\end{equation}
to eventually obtain equation \eqref{eq:extended-spectral-param}.

\subsection{Expressiveness of a Triadic Multilayer Perceptron}
\label{sec:expressiveness_TMLP}

In this section, we examine the expressivity of the triadic \textsc{mlp} and establish a universal approximation result for the proposed architecture.

Let us consider a neural network layer with propagation rule from input $\bm{x}$ to output $\bm{y}$ given by Eq \eqref{eq:2-simplex}. For sake of clarity and without loss of generality, the input bias term is now (and just in this section) incorporated in the triadic interaction. The following expression thus holds:
\begin{equation}
    \label{eq:honn}
    y_k = \sum_i w_{ki}x_i + \sum_{i\leq j} \tilde w_{kij}x_ix_j
\end{equation}
Further bias terms of subsequent triadic layers will not be considered, just as exposed in Sec. \ref{sec:higher_order_layers}. 
In short, we will denote the above transformation with the compact notation $\bm{y} = \text{HONN} (\bm{x})$.

According to Eq. \eqref{eq:extended-spectral-param}, impose the \textit{spectral parametrization} of the weights $\{w_{ki}\}$ and $\{\tilde w_{kij}\}$ amounts to write:
\begin{equation}
\label{eq:spectral_parametrization}
    \begin{aligned}
    w_{ki} &= \left(\lambda_{i}^{\text{(in)}}-\lambda_{k}^{\text{(out)}}\right) \phi_{ki}\\
    \tilde w_{kij} &= \left(\tilde\lambda_{ij}^{\text{(in)}}-\tilde\lambda_{k}^{\text{(out)}}\right)\tilde \phi_{kij},
    \end{aligned}
\end{equation}
where 
\begin{equation}
    \begin{aligned}
        &\lambda_{i}^{\text{(in)}}, \lambda_{k}^{\text{(out)}}, \phi_{ki},\tilde\lambda_{ij}^{\text{(in)}}, \tilde\lambda_{k}^{\text{(out)}} \text{ are tunable},\\
        &\tilde \phi_{kij}\, \text{ are fixed to the initialization values.}
    \end{aligned}
    \label{eq:spectral_constraints}
\end{equation}
As already mentioned, this choice enables one to effectively reduce the number of tunable parameters, which now scales as $O(N^2)$. In essence, training the spectral version of an higher order neural network of the type defined by Eq. (\ref{eq:honn}) comes with a computational cost (as quantified by parameters' scaling) comparable to that associated with standard \textsc{mlp}. Starting from this setting we will prove the following Theorem.

\begin{theorem}
\label{th:universal_approximation_th}
Given $\bm x \in X \subset  \mathbb{R}^n$, the space $\mathcal{H}$ of functions $h(1, \bm{x})$, defined by iteratively composing \textsc{honn} layers of the form of Eq.~\eqref{eq:honn} with the parametrization defined by Eqs.~\eqref{eq:spectral_parametrization},~\eqref{eq:spectral_constraints}, is dense in $C(X, \mathbb{R}^n)$, namely the space of real-valued continuous functions on the compact set $X$.
\end{theorem}

Leveraging the celebrated \textit{Stone-Weierstrass} Theorem, the proof of Theorem \ref{th:universal_approximation_th} can be articulated by showing that:
\begin{equation}
    \mathcal{P}_m(x_1, \cdots, x_n) \subseteq \mathcal{H}\, ,
\end{equation}
where $\mathcal{P}_m(x_1, \cdots, x_n)$ stands for a generic $m$-degree polynomial in  $\mathbb{R}^n$.

In Appendix \ref{sec:direct_space_universal_theorem} 
we will present also a version of this result in  direct space, namely without the constraints imposed by Eqs \eqref{eq:spectral_parametrization}-\eqref{eq:spectral_constraints}. In this latter setting, the architecture aligns with those outlined in \cite{chrysos2021deep, chrysos2022augmenting}. While this result may be intuitive, to the best of our knowledge, no formal proof has yet been documented in the literature.

Before demostrating Theorem \ref{th:universal_approximation_th} in its full generality, we first outline the proof strategy in a simplified case where the input is restricted to a scalar, one-dimensional variable, namely $\mathbf{x} = (x) \in \mathbb{R}$, supplemented with a trivial bias term, $1$. The proof proceeds by iteratively constructing polynomials of increasing degree through the composition of successive layers in the \textsc{honn}.

The first step (see Lemma~\ref{th:1D}) is devoted to the creation of polynomials of second degree. Let us remember that we start with $1$, $x$ (and their combinations), hence we pass from a first degree polynomial to polynomials with degrees one step higher, so to increase expressivity. It should be noted that we require these second-order polynomials to be capable of representing {\it any} arbitrary second-order polynomial.
Our goal is to demonstrate the existence of a parameter configuration for a \textsc{honn} with three output nodes which is capable to generate three linearly independent second-order polynomials, denoted as $y_0(x)$, $y_1(x)$ and $y_2(x)$. The sought claim will be hence indirectly proven. Indeed we will show that a second set of weights, formally a new layer added to the \textsc{honn}, exists such that the standard basis elements of $\mathcal{P}_2(x)$, i.e., $1$, $x$ and $x^2$, can be obtained from linear combinations of $y_j(x)$ that make use of the above mentioned weights (see Fig.~\ref{fig:simple_expressivity_spectral_1D} for a pictorial representation of the claim). Once this step is achieved, we can iterate the construction by building polynomials of degree three upon adding extra layers, and so forth up to any desired degree.

Let us conclude this introduction with two remarks. First, the construction here presented can be extended to build polynomials of arbitrary degree in $n$ variables (see Fig.~\ref{fig:simple_expressivity_spectral} for a pictorial representation and Lemma~\ref{th:add_one_element_to_base} and~\ref{th:polinomial_generation} for a formal proof). Second, the mechanism by which we prove expressivity does not necessarily reflect the configuration produced by an optimization process acting directly on the \textsc{honn} weights. This, however, poses no limitation, as our primary objective is to establish the existence of such expressivity rather than the recovery of the "optimal" parameter setting.

In the case of input $(1, x)$, the higher-order nonlinear interaction generates a term that scales like $x^2$, in the output function $\bm y = \text{HONN}(1,x)$. In the following we will show, that the spectral parameters governing the update rule can be assigned so to map  
$\bm y$  into a basis for $\mathcal{P}_2(x) = \text{Span}\{1,x,x^2\}$.

\begin{figure}[t]
    \centering
    \input{simple_expressivity_spectral_1D}
    \caption{Graphical sketch of Lemma \ref{th:1D}, illustrating the inter-layer connections $(1, x) \to \bm y$ described in Eq.~\eqref{eq:associate_matrix_1d_x^2}, under the triangular ansatz and in the simple setting where the only nonlinear interactions arise from the last column of matrix $W$. Subsequently, the elements of $\bm y$ can be linearly remapped to the canonical basis of $\mathcal{P}_2(x)$. The black dashed arrows represent the linear passage corresponding to the first sum term of Eq. \eqref{eq:honn}, while the orange solid lines represent the higher-order interactions, namely second sum in Eq. \eqref{eq:honn}, here just restricted to $x$ and itself.}
    \label{fig:simple_expressivity_spectral_1D}
\end{figure}
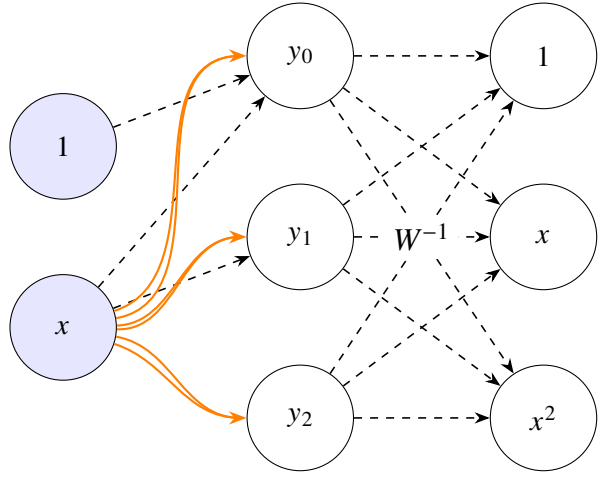

\begin{lemma}
\label{th:1D}
Consider the output of a spectral layer of the form stipulated by Eq.~\eqref{eq:honn}, with the parametrization given by Eqs.~\eqref{eq:spectral_parametrization}-\eqref{eq:spectral_constraints} applied to a one-dimensional input $x$ (together with the bias term), namely:
\begin{equation}
    \bm y = \text{HONN}(1,x),
\end{equation}
where $\bm y = (y_0,y_1,y_2)$. Hence, a solution exists (in terms of the trained spectral parameters) such that $\mathcal{P}_2(x) = \text{Span}\{1,x,x^2\} \subseteq \text{Span}\{y_0,y_1,y_2\},$ i.e., the elements $y_i \in \mathcal{P}_2(x)$ are linearly independent.
\end{lemma}

\begin{proof}
We begin by writing the explicit expressions for $y_0,y_1,y_2$:
\begin{equation}
\begin{aligned}
\label{eq:input_ouput_1D_x^2}
y_0 &= {(w_{00}+\tilde w_{000})}\,1 + (w_{01}+\tilde w_{001})\,x + \tilde w_{011}\,x^2,\\
y_1 &= {(w_{10}+\tilde w_{100})}\,1 + {(w_{11}+\tilde w_{101})}\,x + \tilde w_{111}\,x^2,\\
y_2 &= {(w_{20}+\tilde w_{200})}\,1 + {(w_{21}+\tilde w_{201})}\,x + {\tilde w_{211}\,x^2}.
\end{aligned}
\end{equation}
From Eq.~\eqref{eq:input_ouput_1D_x^2} we can define the associated matrix $W$ acting from the space $\text{Span}\{1,x,x^2\}$ to $\text{Span}\{y_0,y_1,y_2\}$:
\begin{equation}
\label{eq:associate_matrix_1d_x^2}
W=
\begin{pmatrix}
 {w_{00}+\tilde w_{000}} & w_{01}+\tilde w_{001} & \tilde w_{011}\\
 {w_{10}+\tilde w_{100}} & {w_{11}+\tilde w_{101}} & \tilde w_{111}\\
 {w_{20}+\tilde w_{200}} & {w_{21}+\tilde w_{201}} & {\tilde w_{211}}
\end{pmatrix}.
\end{equation}
Imagine that one can find a (at least one) specific combination of the weights, under the imposed spectral parametrization, such that  $\det(W)\ne0$. Then, we can prove that $\{y_0,y_1,y_2\}$ generates the space $\mathcal{P}_2(x)$, as claimed, and consequently invert the above matrix to remap the output into the canonical basis through a linear layer whose weights incorporate $W^{-1}$ (see Fig. \ref{fig:simple_expressivity_spectral_1D} for a visual representation of the proposed procedure).

To reach this goal, we show that 
$W$ can be forced into an upper triangular form, with non zero elements across the main diagonal, which in turn amounts to set out  
$\det(W) \ne 0$, as wished.
More concretely, we will demonstrate that the spectral parametrization is sufficiently flexible to turn to zero all the entries under the diagonal of matrix \eqref{eq:associate_matrix_1d_x^2},  while keeping the diagonal terms nonzero.

In formulae:
\begin{equation}
    \begin{aligned}
        &{w_{10}=-\tilde w_{100}}\implies \left(\lambda_{0}^{\text{(in)}}-\lambda_{1}^{\text{(out)}}\right) \phi_{10} = -\tilde w_{100} \implies \phi_{10} = -\frac{\tilde w_{100}}{\lambda_{0}^{\text{(in)}}-\lambda_{1}^{\text{(out)}}}, \: \lambda_{0}^{\text{(in)}}\ne\lambda_{1}^{\text{(out)}},\\
        &{w_{20}=-\tilde w_{200}}\implies \left(\lambda_{0}^{\text{(in)}}-\lambda_{2}^{\text{(out)}}\right) \phi_{20} = -\tilde w_{200} \implies \phi_{20} = -\frac{\tilde w_{200}}{\lambda_{0}^{\text{(in)}}-\lambda_{2}^{\text{(out)}}}, \: \lambda_{0}^{\text{(in)}}\ne\lambda_{2}^{\text{(out)}},\\
        &{w_{21}=-\tilde w_{201}}\implies \left(\lambda_{1}^{\text{(in)}}-\lambda_{2}^{\text{(out)}}\right) \phi_{21} = -\tilde w_{201} \implies \phi_{21} = -\frac{\tilde w_{201}}{\lambda_{1}^{\text{(in)}}-\lambda_{2}^{\text{(out)}}}, \: \lambda_{1}^{\text{(in)}}\ne\lambda_{2}^{\text{(out)}},\\
    \end{aligned}
    \label{eq:conditions_under_triangle_1d}
\end{equation}
and
\begin{equation}
    \begin{aligned}
        &{w_{00}\ne-\tilde w_{000}}\implies \left(\lambda_{0}^{\text{(in)}}-\lambda_{0}^{\text{(out)}}\right) \phi_{00} \ne -\tilde w_{000} \implies \phi_{00} \ne -\frac{\tilde w_{000}}{\lambda_{0}^{\text{(in)}}-\lambda_{0}^{\text{(out)}}},\\
        &{w_{11}\ne-\tilde w_{101}}\implies \left(\lambda_{1}^{\text{(in)}}-\lambda_{1}^{\text{(out)}}\right) \phi_{11} \ne -\tilde w_{101} \implies \phi_{11} \ne -\frac{\tilde w_{101}}{\lambda_{1}^{\text{(in)}}-\lambda_{1}^{\text{(out)}}},\\
        &{\tilde w_{211}\ne 0} \implies \left(\tilde\lambda_{11}^{\text{(in)}}-\tilde\lambda_{2}^{\text{(out)}}\right)\tilde \phi_{211} \ne 0 \implies \tilde\lambda_{2}^{\text{(out)}} \ne \tilde \lambda_{11}^{\text{(in)}}\:\: \text{and}\:\: \tilde\phi_{211} \ne 0.
    \end{aligned}
    \label{eq:conditions_diagonal_1d}
\end{equation}

The conditions stemming from  Eqs.~\eqref{eq:conditions_under_triangle_1d} and \eqref{eq:conditions_diagonal_1d} can be straightforwardly fulfilled by (i) setting $\tilde\phi_{211} \ne 0$. Notice that $\tilde\phi_{211}$ is a non tunable parameter and thus it gets permanently frozen to the value assigned following initialization.
(ii) exploiting the fact that $\phi_{ij}$ can be tuned at will to set them to the values prescribed in Eq.~\eqref{eq:spectral_constraints}.
Under the above mentioned conditions, matrix $W$ is reduced to a upper triangular form with non zero diagonal entries. As such, it can be inverted which immediately yields that $\text{Span}\{y_0,y_1,y_2\} = \mathcal{P}_2(x)$, namely the proof of the Lemma.
\end{proof}

The proof can be adapted easily to deal with the subsequent step that aims at generating $\mathcal{P}_3(x)$ from  $\mathcal{P}_2(x)$, with the inclusion of one additional cubic term. By iterating forward the process, with the inclusion of just one non-linear monomial per step, it readily yields $\mathcal{P}_m(x)$, for any generic $m$ degree. Instead of detailing the derivation of the above results with reference to the scalar setting, we shall turn to discuss the general case of polynomial $\mathcal{P}_m(\bm{x})$, where $\bm{x} \in  \mathbb{R}$ as addressed in Theorem \ref{th:universal_approximation_th}.

First, we start by computing the number $D$ of monomials which are associated to a complete polynomial of degree $m$ in $\mathbb{R}^n$. This is given by:

\begin{equation}
    \label{eq:hockey_stick_identity}
    D = \sum_{k=0}^m\binom{n+k-1}{k}= \sum_{k=0}^m\binom{n+k-1}{n-1} = \sum_{k=n-1}^{n+m-1}\binom{k}{n-1} = \binom{m+n}{n},
\end{equation}
where in the last chain of identities of Eq.~\eqref{eq:hockey_stick_identity} we used the \textit{hockey stick identity}. 

Then, to prove the Theorem we proceed with a \textit{ladder} strategy that we outline in the following. Given $\{\underline{e}_0, \cdots, \underline{e}_q \}$, $q+1$ elements of the canonical basis of $\mathcal{P}_m(x_1, \cdots, x_n)$, we provide an iterative recipe to yield a novel and independent element $\underline{e}_{q+1}$ until reaching the sought dimension $D$.

The first step of the ladder strategy (or the initial level of an induction proof) refers to the input features (potentiated with the bias term): $1,x_1,\cdots, x_n$, which also belong to the canonical basis of $\mathcal{P}_m(x_1, \cdots, x_n)$. The following Lemma holds:

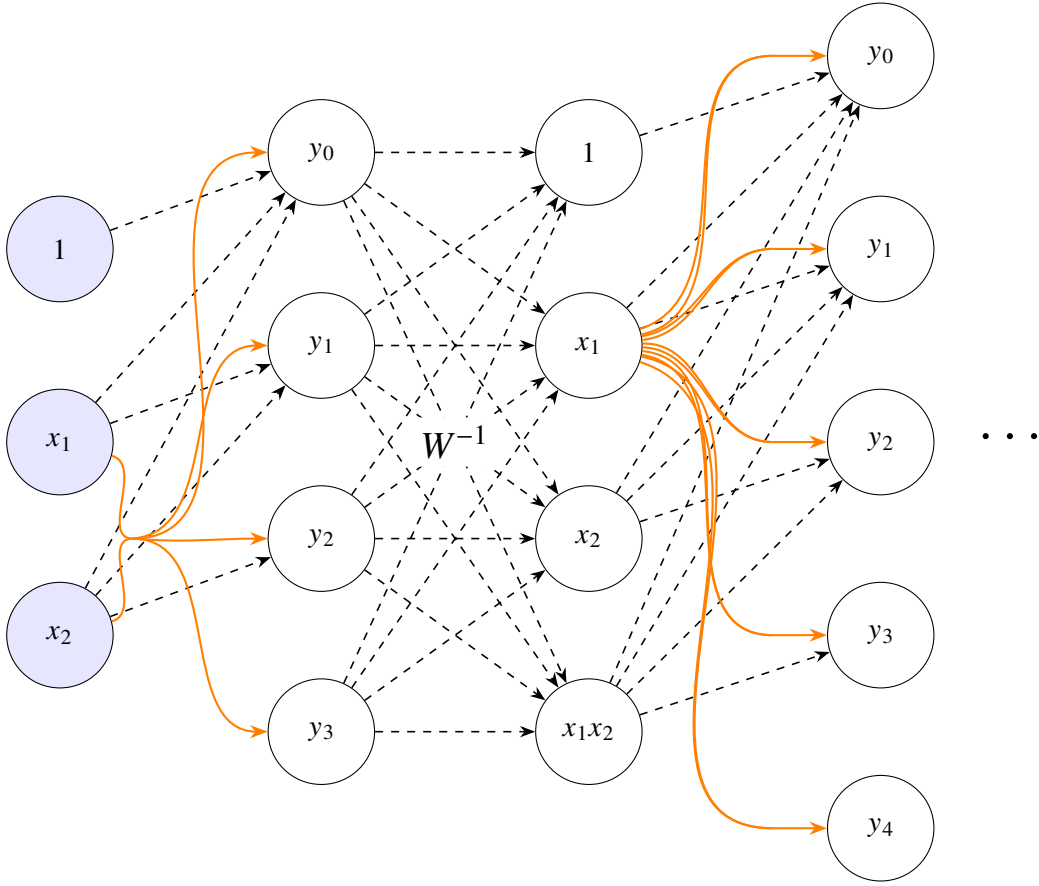
\begin{figure}[t]
    \centering
    \input{simple_expressivity_spectral_bias}
    \caption{Graphical sketch of the ladder demonstration scheme proposed in Lemmas \ref{th:add_one_element_to_base}, \ref{th:polinomial_generation}. Here the input $(1,x_1,x_2)$ is provided to the \textsc{honn} with output $(y_0,y_1,y_2,y_3)$, which can be linearly remapped, via $W^{-1}$, to the new canonical basis augmented with a new nonlinear term, for instance $x_1x_2$. The procedure then continues by creating new terms: $x_1^2$, $x_2^2$, and so on. As in Fig.~\ref{fig:simple_expressivity_spectral_1D}, dashed black lines indicate linear transfer, while orange ones indicate nonlinear interactions.}
    \label{fig:simple_expressivity_spectral}
\end{figure}

\begin{lemma}
{\bf The first step of the induction proof of Theorem \ref{th:universal_approximation_th}}    \label{th:add_one_element_to_base} \\
    Given $(1, \bm x) = (1, x_1, \cdots, x_n)$, then there exists a set of weights $w_{ki}$ and $\tilde w_{kij}$ parametrized according to Eqs.~\eqref{eq:spectral_parametrization},~\eqref{eq:spectral_constraints} such that $\text{Span}\{1,x_1, \cdots,x_n, x_{\bar h}x_{\bar k}\} \subseteq\text{Span} \{y_0, \cdots,y_{n+1}\}$ where $ \bm y = HONN(1, \bm x)$.
\end{lemma}

\begin{proof}
Consider $\bm y = (y_0, y_1, \cdots, y_{n+1}) = \text{HONN}(1,\bm x)$, where $\text{HONN}$ stands for a neural network layer defined by Eq.~\eqref{eq:honn} with $n+1$ nodes, as the input and $n+2$ nodes for the output. Restrict to the case where the just one non-linear term is produced (at each iteration or, equivalently, added layer) by the interaction among $x_{\bar h}$ and $x_{\bar k}$, for any generic $\bar h$ and $\bar k$ (raging from $1$ to $n$). Then, the general output $\bm y$ can be written as:
\begin{equation}
\label{eq:general_output_general_dim}
\begin{aligned}
y_0 &= (w_{00}+\tilde w_{000})1 + \cdots + (w_{0j}+\tilde w_{00j})x_j + \cdots + \tilde w_{0\bar h \bar k}x_{\bar h}x_{\bar k} ,\\[2pt]
y_1 &= (w_{10}+\tilde w_{100})1 +  \cdots + (w_{1j}+\tilde w_{10j})x_j + \cdots + \tilde w_{1 \bar h \bar k}x_{\bar h}x_{\bar k} ,\\[2pt]
&\vdots\\
y_i &= (w_{i0}+\tilde w_{i00})1 +  \cdots + (w_{ij}+\tilde w_{i0j})x_j + \cdots + \tilde w_{i\bar h \bar k}x_{\bar h}x_{\bar k} ,\\[2pt]
&\vdots\\
y_{n+1} &= (w_{n+1,0}+\tilde w_{n+1,00})1 +  \cdots + (w_{n+1,j}+\tilde w_{n+1,0j})x_j + \cdots + \tilde w_{n+1,\bar h \bar k}x_{\bar h}x_{\bar k} ,\\[2pt]
\end{aligned}
\end{equation}
with an associated matrix, connecting $\bm y$ to the new basis $\{1, x_1, \cdots,x_n, x_{\bar h} x_{ \bar k} \}$, given by:
\begin{equation}
W=
\begin{pmatrix}
w_{00}+\tilde w_{000} & \cdots &w_{0j}+\tilde w_{00j} & \cdots & \tilde w_{0\bar h \bar k}\\
\vdots &  & \cdots && \cdots \\
w_{i0}+\tilde w_{i00} & \cdots & w_{ij}+\tilde w_{i0j} & \cdots & \tilde w_{i\bar h \bar k}\\
\vdots &  & \cdots && \cdots\\
 w_{n+1,0}+\tilde w_{n+1,00} & \cdots & w_{n+1,j}+\tilde w_{n+1,0j}&\cdots& \tilde w_{n+1,\bar h \bar k}
 \end{pmatrix}.
 \label{eq:associate_matrix_general_demonstration}
 \end{equation}

As shown with reference to the simple scalar setting examined above, it is sufficient to show that a set of spectral weights exists such that $\det(W)\ne 0$. This ensures that $\{y_0, \cdots,y_{n+1}\}$ generates $\text{Span}\{1, x_1, \cdots,x_n, x_{\bar h}x_{\bar k} \}$, namely that the produced output can be remapped onto the expanded canonical basis via a linear transformation associated to $W^{-1}$ (see Fig. \ref{fig:simple_expressivity_spectral}).  

Following the reasoning outlined above, we shall set to prove that 
a specific choice of the parameters exists that turns $W$ into an upper triangular matrix, with non zero diagonal elements. Specifically,  we seek to set to zero all the terms $w_{ij},\, i>j$, while imposing $w_{ij} \ne 0, \, i=j.$ To this end write: 

\begin{equation}
    \begin{aligned}
        &w_{ij}+\tilde w_{i0j} = 0, \: \text{if} \:\:i>j,\\
        &w_{ij}+\tilde w_{i0j} \ne 0, \: \text{if} \:\:i=j,\\
        &\tilde w_{n+1,\bar h \bar k} \ne 0.
    \end{aligned}
    \label{eq:conditions_for_triangular_general}
\end{equation}
The first of the above conditions, expressed in terms of the spectral parametrization Eq.~\eqref{eq:spectral_parametrization}, yields:
\begin{equation}
    \left(\lambda_{j}^{\text{(in)}}-\lambda_{i}^{\text{(out)}}\right) \phi_{ij} = -\tilde w_{i0j} \implies \phi_{ij} = -\frac{\tilde w_{i0j}}{\lambda_{j}^{\text{(in)}}-\lambda_{i}^{\text{(out)}}},  \:\lambda_{j}^{\text{(in)}}\ne\lambda_{i}^{\text{(out)}}, \: \forall\, i>j,
    \label{eq:spectral_condition_under_triangle}
\end{equation}
The second set of conditions, prescribed in Eq. \eqref{eq:conditions_for_triangular_general}, results in
\begin{equation}
    \left(\lambda_{j}^{\text{(in)}}-\lambda_{i}^{\text{(out)}}\right) \phi_{ij} \ne -\tilde w_{i0j} \implies \phi_{ij} \ne -\frac{\tilde w_{i0j}}{\lambda_{j}^{\text{(in)}}-\lambda_{i}^{\text{(out)}}}, \: \forall\, i=j.
    \label{eq:spectral_condition_on_diagonal}
\end{equation}
Both requirements~\eqref{eq:spectral_condition_under_triangle},~\eqref{eq:spectral_condition_on_diagonal} can be easily fulfilled since the spectral eigenvector entries $\phi_{ij}$ (together with the associated eigenvalues) are fully trainable.

The last condition of Eq.~\eqref{eq:conditions_for_triangular_general} reads,
\begin{equation}
    \tilde w_{n+1,\bar h \bar k}\ne 0 \implies \left(\tilde\lambda_{\bar h \bar k}^{\text{(in)}}-\tilde\lambda_{n+1}^{\text{(out)}}\right)\tilde \phi_{n+1,\bar h \bar k} \ne 0 \implies \tilde\lambda_{n+1}^{\text{(out)}} \ne \tilde \lambda_{\bar h \bar k}^{\text{(in)}}.
\end{equation}
This is the only condition imposed on the trainable higher-order spectral parameters (together with initializing $\tilde \phi_{n+1,\bar h \bar k} \ne 0$) and, as such, it can be accommodated for without clashes in the process of nested parameters handling. This ends the proof of the Lemma.
\end{proof}

As we can see from Eq. \eqref{eq:spectral_condition_under_triangle} the parameters $w_{i00}$ are free to be set to zero, recovering the numerical setup where the biases are not included in the higher-order sum of Eq. \eqref{eq:honn}. 

We can proceed forward to prove the following Lemma, which constitutes the main inductive step of the proof. In turn, the following proof combined with the preceding Lemma, will  return a complete proof of the sought Theorem \ref{th:universal_approximation_th}.

\begin{lemma} {\bf The general step of the induction proof of Theorem \ref{th:universal_approximation_th}}       \label{th:polinomial_generation}

    Given $\{\underline{e}_0, \cdots, \underline{e}_q \}$, with $q < D-1$ and  $\underline{e}_0 \equiv 1,\, \underline{e}_1 \equiv x_1,\, \cdots,\, \underline{e}_n \equiv x_n,$. Then there exists a set of weights $w_{ki}$ and $\tilde w_{kij}$ parametrized according to Eqs.~\eqref{eq:spectral_parametrization},~\eqref{eq:spectral_constraints} such that $\text{Span}\{1,\underline{e}_1, \cdots, \underline{e}_q, \underline{e}_{\bar h}\underline{e}_{\bar k}\} \subseteq\text{Span} \{y_0, \cdots,y_{q+1}\}$ where $ \bm y = HONN(1,  \underline{\bm e})$.
\end{lemma}

\begin{proof}
The scheme applied to the proof of Lemma \ref{th:add_one_element_to_base} can be straightforwardly generalized to deal with a set of monomials  $\{\underline{e}_0, \cdots, \underline{e}_q \}$, with $q < D-1$. In particular, $\underline{e}_0 \equiv 1,\, \underline{e}_1 \equiv x_1,\, \cdots,\, \underline{e}_n \equiv x_n,$ and define $\underline{e}_{q+1} \equiv e_{\bar{h}} e_{\bar{k}}$, for any integers $\bar{h}$ and $\bar{k}$ (smaller than $q$). It is immediate to show that this setting maps exactly onto the one analyzed above in Lemma \ref{th:add_one_element_to_base}. In other words, the produced output can be turned into a canonical basis via a linear transformation that is granted by the inverse of an upper triangular matrix with non zero determinant. One can thus generate, any monomial $\bar{q}$ in the range $q < \bar{q} < D$, as an iterative sequence of the algorithm described above. 
\end{proof}

We have thus proved by induction that a basis for $\mathcal{P}_m(x_1, \cdots, x_n)$ can be iteratively constructed by an ad hoc choice of the spectral parameters, thus implying that, sufficiently large, spectral higher order neural networks of the type here considered,  possess the  expressivity required to approximate arbitrary polynomials of any degrees $m$ and dimensions $n$. An illustrative scheme to depict the strategy employed in the proof is displayed in Figure \ref{fig:simple_expressivity_spectral}. 
The function $h(\bm x)$, as obtained by stacking an arbitrary number of layers (with the needed number of nodes to store the basis elements), can thus mimic a general polynomial of arbitrary degree. Therefore, thanks to the \textit{Stone-Weierstrass theorem}, it can approximate a target continuous function defined on a compact set to the desired level of approximation.

\newpage

\appendix
\section{Proof of the expressiveness in Direct Space}
\label{sec:direct_space_universal_theorem}

In the following we shall denote by $\mathcal{P}_k(\mathbb{R}^n)$ the vector space of all polynomials of $n$ variables, up to $k$-th degree; we will also use the notation $m^{(h)}$ to refer to an element of degree $h$, of \textit{the standard basis}, \textit{i.e.}, composed by the standard monomials of $\mathcal{P}_k(\mathbb{R}^n)$.

Before delving into the general theorem on the expressiveness of (non-spectral) deep triadic \textsc{mlp}s, it is instructive to consider the basic case of a triadic network with just one input neuron. Given a scalar input $x$ as a first key observation we note that the forward pass, see Eq.~\eqref{eq:2-simplex} in the main text, allows the model to output the standard basis of the space $\mathcal{P}_2(\mathbb{R})$, in the second layer (provided enough neurons are supplied). Indeed, if one equips the second layer with $3$ neurons, it is manifestly possible to set the weights $w_{ki},\tilde{w}_{kij}$ so as to obtain output (proportional to) $1$ on the first neuron of the second layer, (proportional to) $x$ on the second, and (proportional to) $x^2$ on the third. We have thus recovered the standard basis of $\mathcal{P}_2(\mathbb{R})$. 

The second key observation consists in noticing that this process can be iterated across subsequent layers: given a neural layer containing the standard basis of the space $\mathcal{P}_{k}(\mathbb{R})$, the forward pass that follows Eq.~\eqref{eq:2-simplex} can generate the standard basis of the space $\mathcal{P}_{2k}(\mathbb{R})$ in the next layer (once again, if enough neurons are present). The third, and final, observation is that, given the standard basis of $\mathcal{P}_{k}(\mathbb{R})$, the forward pass dictated by Eq.~\eqref{eq:2-simplex} is clearly able to output any polynomial function at least up to degree $k$ in the successive layer, no matter the size of the target output vector. These three observations, once put together, clearly imply that a single-input triadic \textsc{mlp}, if large enough, can regress any polynomial function of $x$. By leveraging on the \textit{Stone–Weierstrass} theorem, we can thus conclude that a sufficiently large, single-input, triadic \textsc{mlp} can approximate \textit{any} continuous function of $x$, even in the absence of a non-linear activation function. This result amounts to an universal approximation theorem for triadic neural networks with one-dimensional inputs, in direct space. A simple graphical sketch of the provided arguments is displayed in Fig.~\ref{fig:simple-expressivity}.

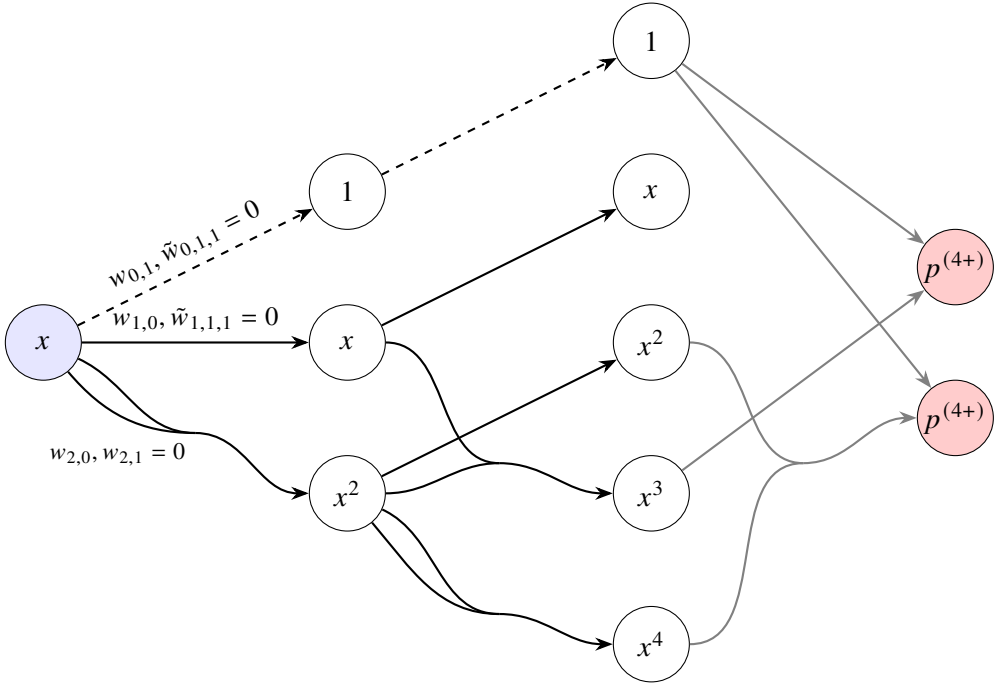
\begin{figure}[t]
    \centering
    \input{simple_expressivity}
    \caption{Graphical sketch of the universality proof for triadic multilayer perceptrons. The structure of the forward pass \eqref{eq:2-simplex} allows our network to build standard basis of the polynomial vector spaces. Specifically, in this toy example it is easy to see that a sufficiently large $k$-th layer can yield the standard basis of the vector space $\mathcal{P}_{2^{k-1}}(\mathbb{R})$. This in turn implies that the subsequent layer will be able to generate polynomial functions $f:\mathbb{R} \to \mathbb{R}^{n}$ at least up to the $2^{k-1}$ order. We keep track of the inherited level of expressivity with the $p^{(\circ +)}$ node tag (in the example in figure $k=3$).}
    \label{fig:simple-expressivity}
\end{figure}

The above qualitative description can be put on solid mathematical ground thanks to Theorem~\ref{UATTN}, which holds true in the general setting of $n$ variables. As a preliminary result, let us state and prove the following Lemma:
\begin{lemma}
A triadic perceptron, with a single output neuron, processing as an input the standard base of the vector space $\mathcal{P}_{k}(\mathbb{R}^n)$ can output any monomial $m$ up to order $2k$.
\label{lemma}
\end{lemma}

\begin{proof}
The proof is established by contradiction. Let us focus on the single neuron of the next layer: suppose it exists a monomial $\tilde{m}$ of order $\alpha \leq 2k$ that cannot be generated by the model forward pass, \textit{i.e.} cannot be written in the form
\begin{equation}
    \sum _{i} w_{0i}m_i^{(k)} + \sum _{i,j} \tilde{w}_{0ij}m_i^{(k)}m_j^{(k)},
\end{equation}
and so it does not belong to the span:
\begin{equation}
    \tilde{m} \notin \text{span} \left(\{m^{(k)}_i\} \cup \{m^{(k)}_im^{(k)}_j\}\right),
    \label{eq:span}
\end{equation}
where $m_i^{(k)}$ and $m_j^{(k)}$ are two monomials of degree $k$ (note that the indicies $i,j$ span all the order $k$ standard basis monomials).
However, note that one can always write the selected monomial as:
\begin{equation}
    \tilde{m} = m' \cdot m'',
    \label{eq:before_passage}
\end{equation}
where $m'$ and $m''$ are two monomials of order $\alpha ', \alpha '' \leq k$.  This is clearly an element of the span of $\{m^{(k)}_im^{(k)}_j\}$, thus leading to a contradiction.
\end{proof}

\begin{theorem}
    A sufficiently large triadic multilayer perceptron is able to approximate any continuous function $f:\mathbb{R}^n \to \mathbb{R}^\nu$ with an arbitrary degree of precision.
    \label{UATTN}
\end{theorem}
\begin{proof}
Let us start by proving that a large enough triadic multilayer perceptron, with input $\bm{x} \in \mathbb{R}^n$, can generate a standard basis for any polynomial vector space $\mathcal{P}_k(\mathbb{R}^n)$, regardless of how large $k$ is.

This proof is by induction: the base case consists in observing that a triadic perceptron, with input layer of size $v$, by definition holds in its first layer a standard basis of the vector space $\mathcal{P}_{1}(\mathbb{R}^n)$ (\textit{i.e.} $x,y,z,...$; and as we argued in previous sections the absence of the constant bias term as neural activation is irrelevant).

For the inductive step we want to show that, given a collection of monomials $\{m_i^{(k)}\}$ that forms the standard basis for the vector space $\mathcal{P}_{k}(\mathbb{R}^n)$, a triadic perceptron gathering such a collection as an input is able to output, at the next layer, the standard basis of the vector space $\mathcal{P}_{2k}(\mathbb{R}^n)$ (once again, assuming the next layer has enough neurons). Thanks to Lemma \ref{lemma}, we know that the proposed forward pass is able to generate any monomial up to order $2k$ as output. Since the lemma applies to each next layer neuron (with new weights each time), it is sufficient to prove that the triadic perceptron can to generate any collection of monomials up to order $2k$ as an output, \textit{i.e.} it is able to generate the standard basis of the polynomial vector space $\mathcal{P}_{2k}(\mathbb{R}^n)$, thus concluding the proof of the inductive step.

A sufficiently large triadic multilayer perceptron is able to generate any  standard basis $\{m^{(s)}_i\}$ of the polynomial vector space, regardless of the value of $s$ (maximum order) or $n$ (number of variables, \textit{i.e} input size).  This in turn implies that, by adding a new neural layer on top of the existing ones, we will be able to represent any polynomial function $p:\mathbb{R}^n \to \mathbb{R}^\nu$ up to order $s$, with $\nu$ size of the lastly added layer.

Given that there is no bound to the order $s$ (it can be made as large as wished by adding new layers), this means that a sufficiently large triadic multilayer perceptron is able to represent any polynomial function, of any order. This at last allows us to invoke the \textit{Stone–Weierstrass} theorem to finally prove that the model, if given enough neurons in enough layers, is capable to approximating, with an arbitrary degree of precision, any continuous function.
\end{proof}

\section{Parameter scaling comparison with state-of-the-art models}\label{sec:parameters_scaling}

The novel hypergraph architecture can be manifestly re-interpreted as a deep polynomial architecture. In the literature similar models have been discussed \cite{chrysos2021deep, chrysos2022augmenting}, but remarkably the proposed spectral variant yields a smaller parameter scaling when compared with other state-of-the-art polynomial architectures, as we shall see in this section.

For a perceptron with $N$ input and output neurons, a spectral triadic network has the following parametrization
\begin{equation}
    y_k = \sum _{i} w_{ki} x_i + \sum _{i \leq j} (\tilde{\lambda} _{ij}^{(in)}-\tilde{\lambda} _k^{(out)})\tilde{\phi}_{kij}x_ix_j,
\end{equation}
note that the output is a polynomial of order $2$. Since the generalized eigenvectors $\tilde{\phi}_{kij}$ are not trained this implies the handling of two parameter matrices ($W, \tilde{\Lambda} ^{(in)}$), and a single parameter vector ($\bm{\tilde{\lambda}}^{(out)}$). The number of involved parameters is thus:
\begin{equation}
    P_{spectral}(N) = \underbrace{N(N+1)}_{W}+ \underbrace{\frac{N(N+1)}{2}}_{\tilde{\Lambda} ^{(in)}}+ \underbrace{N}_{\bm{\tilde{\lambda}}^{(out)}}=\frac{3}{2}N^2+\frac{5}{2}N.
\end{equation}
If instead we choose to parametrize spectrally also the linear transfer we get:
\begin{equation}
    y_k = \sum _{i} (\lambda_i^{(in)}-\lambda _k^{(out)})\phi_{ki}x_i + \sum _{i,j} (\tilde{\lambda} _{ij}^{(in)}-\tilde{\lambda} _k^{(out)})\tilde{\phi}_{kij}x_ix_j.
\end{equation}
Notice that the triadic part is unaltered, but instead of parameterizing the linear transfer with a single weight matrix $W$ we are now using one matrix ($\phi$) and two vectors ($\bm{\lambda}^{(in)},\bm{\lambda}^{(out)}$). This, of course, implies the following  parameter scaling:
\begin{equation}
    P_{fullspectral}(N) = \underbrace{3N + 1}_{\bm{\lambda}^{(in)},\bm{\lambda}^{(out)},\bm{\tilde{\lambda}}^{(out)}} + \underbrace{N(N+1)}_{\phi} + \underbrace{\frac{N(N+1)}{2}}_{\tilde{\Lambda} ^{(in)}} = \frac{3}{2}N^2+\frac{9}{2}N+1
\end{equation}
Notably, this resultant scaling remains superior to the most parsimonious deep polynomial architectures currently documented in the literature. This is $P_{ccp}(N) = 3N^2 + N$ for the \textsc{ccp} decomposition version of the \textit{ProdPoly} network \cite{chrysos2021deep}, as elaborated in the next sub-section.

\subsection{CCP Decomposition}

In the CCP decomposition a polynomial of order $2$ is instead obtained via the following:

\begin{equation}
    \bm{y} = \bm{\beta} + CU^T_{[1]}\bm{x} + C(U_{[2]} \odot U_{[1]})^T (\bm{x} \odot \bm{x})
\end{equation}
where $\odot$ indicates the \textit{Khatri-Rao} product. Given, as before, an input dimension and an output dimension equal to $N$, this new parametrization introduces a degree of ambiguity since a new hidden dimension $K$ is present. In fact, $C$ is  a $N \times K$ matrix, and $U_{[i]}$ are all $N \times K$ matrices as well. This new variable could imply an increment or a reduction of the number of trainable parameters. For a fair comparison, we just take $K$ to be equal to $N$, and this yields the following parameter scaling for the \textsc{ccp} decomposition:
\begin{equation}
    P_{ccp}(N) = 3N^2 + N
\end{equation}
The result is indeed expected since we deal with three $N \times N$ matrices ($C,U^T_{[1]},U^T_{[2]}$), and a vector $\bm{\beta}$ of length $N$.

\section{Regression}\label{sec:regression}

\begin{figure}[th!]
    \centering
    \includegraphics[width=0.9\linewidth]{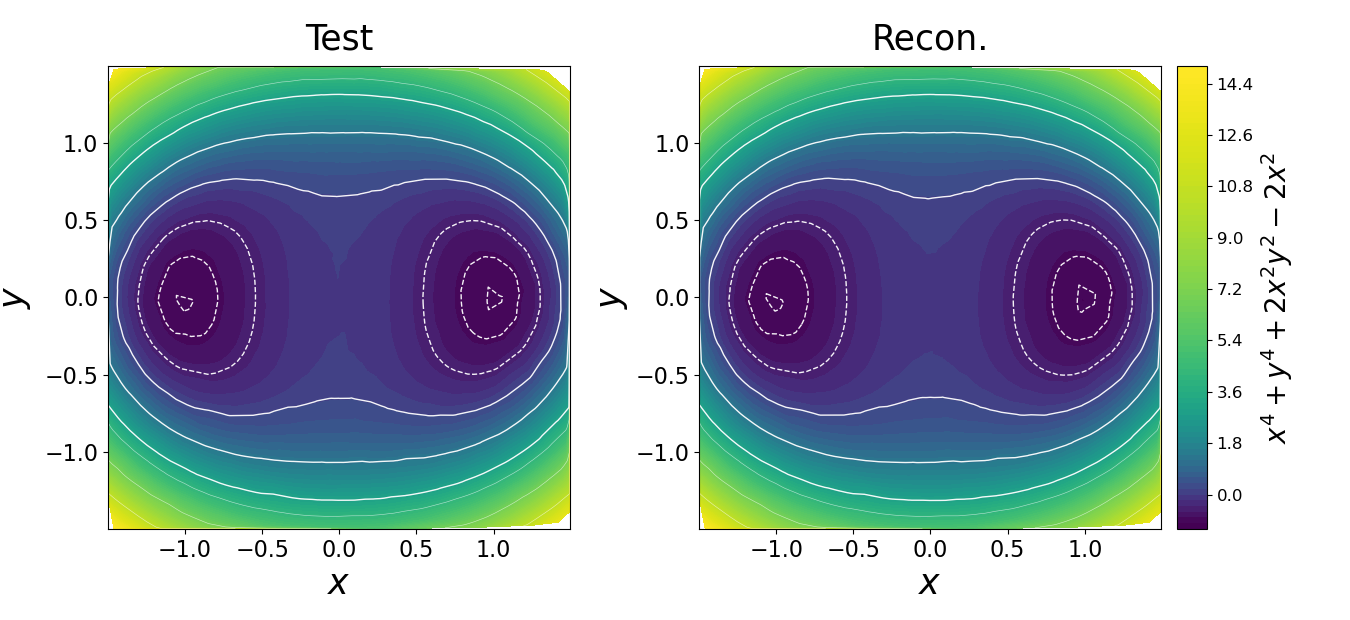}
    \caption{Heatmap comparison between the function $f(x,y) = x^4 + y^4 +2x^2y^2-2x^2$ (on the left) and the reconstructed data by a triadic spectral \textsc{mlp} with hidden dimension $h_{dim}=20$ and $n_\ell = 3$ number of layers (by including input and output) and trained with the \textit{Adam} optimizer and fixed \textit{learning rate}.}
    \label{fig:double-well}
\end{figure}

\begin{figure}[th!]
    \centering
    \subfloat[]{\includegraphics[width=0.49\linewidth]{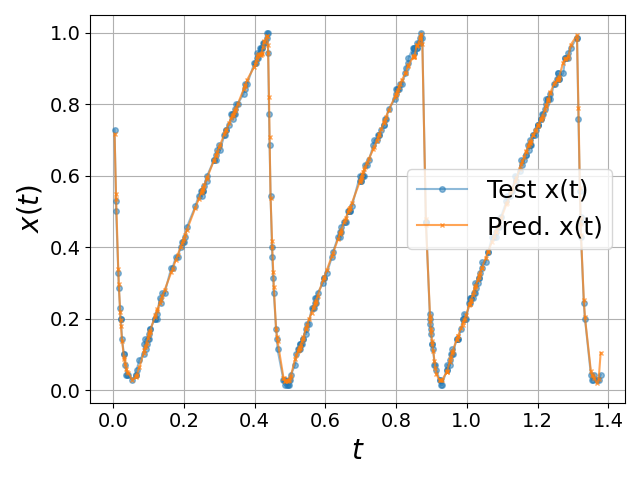}}
    \subfloat[]{\includegraphics[width=0.49\linewidth]{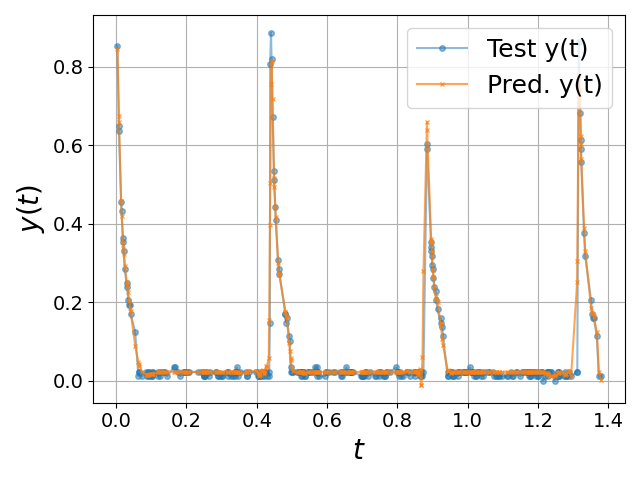}}
    \caption{Regression of experimental data, opportunely rescaled, representing a spiking neural-like behaviour of a relaxation nonlinear circuit \cite{febbe2024dynamical, febbe2024chaos}. On the left, we can see the charge of a capacitor indicated as $x(t)$, fully reconstructed while on the right the  regression of a spiking current denoted by $y(t)$. This fit has been carried out by implementing a triadic \textsc{mlp} with hidden dimension $h_{dim}=20$ and $n_\ell = 20$ number of layers (including input and output) and trained with the \textit{Adam} optimizer, following a \textit{halving-lr-on-plateau} scheduler. Here, self-coupling pairs are not included into the implementation.}
    \label{fig:spiking_dynamical_system}
\end{figure}

Here we report on a gallery of regression tests based on the proposed architectures, both on mock and real data. These results serve primarily as a visual demonstration of the expressive power of the proposed network. A general theorem concerning universality is presented in Sec.~\ref{sec:expressiveness_TMLP}.

We first regressed the double-well profile $f(x,y) = x^4 + y^4 + 2x^2y^2 - 2x^2$ with a triadic spectral \textsc{mlp} (see Fig.~\ref{fig:double-well}). As shown in the figure, the target function (depicted in the left panel) is accurately reconstructed (right panel), with a test \textsc{mae} of $9.5 \times 10^{-6}$.

We also present a fit of experimental data of a two-dimensional dynamical system representing a nonlinear circuit exhibiting neural-like spiking behavior (see \cite{febbe2024dynamical,febbe2024chaos}). As it can be appreciated by visual inspection, a triadic \textsc{mlp} implementation is able to closely track the abrupt behavioral changes as displayed by the recorded data and reproduce the spiking dynamics, with a mean average error \textsc{mae} $8 \times 10^{-3}$.

\FloatBarrier

\subsection*{Code Availability}
The code used for the tests reported in this paper is available at the following link: \url{https://github.com/gianluca-peri/hyperspectral}

\bibliographystyle{ieeetr}
\bibliography{bibliography}

\end{document}

%% file: perceptron1.tex
\begin{center}
\begin{tikzpicture}[
    >=Stealth,
    neuron_style/.style={
        circle, draw=black, minimum size=1.0cm, inner sep=0pt,
        font=\sffamily\normalsize
    },
    input/.style={neuron_style, fill=blue!10},
    output/.style={neuron_style, fill=red!20},
    arrow/.style={->, thick, black},
    dotted_arrow/.style={->, thin, black!60, dashed},
    label/.style={font=\sffamily\small}
]
\def\nodeSpacing{1.2}
\def\inX{0}
\def\outX{3}

\foreach \i in {1,...,5}
    \node[input] (x\i) at (\inX, { (3-\i)*\nodeSpacing }) {$x_\i$};

\node[output] (y1) at (\outX, 1*\nodeSpacing) {$y_1$};
\node[output] (y2) at (\outX, 0*\nodeSpacing) {$y_2$};
\node[output] (y3) at (\outX, -1*\nodeSpacing) {$y_3$};

\draw[arrow, purple] (x1) -- (y1) node[midway, above, sloped, label] {};
\draw[arrow, green!60!black] (x1) -- (y2) node[midway, above, sloped, label] {};
\draw[arrow, orange] (x3) -- (y2) node[midway, below, sloped, label] {};

\draw[dotted_arrow] (x2) -- (y1);
\draw[dotted_arrow] (x2) -- (y2);
\draw[dotted_arrow] (x2) -- (y3);
\draw[dotted_arrow] (x4) -- (y2);
\draw[dotted_arrow] (x5) -- (y2);

\node[
    draw=black,          
    thick,               
    rounded corners=5pt, 
    inner sep=8pt,       
    text height=1.0cm,   
    text depth=1.0cm,    
    fill=black!10,
    minimum width=4cm    
] at (1.5, -4.5) {\footnotesize$\displaystyle y_k = \sum _{i=0}^N w_{ki}x_i$};

\draw[thick, rounded corners=15pt] 
    ([xshift=-0.22cm, yshift=-0.22cm]current bounding box.south west) 
    rectangle 
    ([xshift=0.22cm, yshift=0.22cm]current bounding box.north east);

\end{tikzpicture}
\end{center}

%% file: perceptron2.tex
\begin{center}
\begin{tikzpicture}[
    >=Stealth,
    neuron_style/.style={
        circle, draw=black, minimum size=1.0cm, inner sep=0pt,
        font=\sffamily\normalsize
    },
    input/.style={neuron_style, fill=blue!10},
    output/.style={neuron_style, fill=red!20},
    arrow/.style={->, thick, black},
    dotted_arrow/.style={->, thin, black!60, dashed}
]
\def\nodeSpacing{1.2}
\def\inX{0}
\def\outX{3}
\def\midX{1.5}

\foreach \i in {1,...,5}
    \node[input] (x\i) at (\inX, { (3-\i)*\nodeSpacing }) {$x_\i$};

\node[output] (y1) at (\outX, 1*\nodeSpacing) {$y_1$};
\node[output] (y2) at (\outX, 0*\nodeSpacing) {$y_2$};
\node[output] (y3) at (\outX, -1*\nodeSpacing) {$y_3$};

\draw[dotted_arrow] (x1) -- (y1);
\draw[dotted_arrow] (x3) -- (y2);
\draw[dotted_arrow] (x2) -- (y2);
\draw[dotted_arrow] (x5) -- (y2);

\coordinate (M1) at (\midX, 1.5*\nodeSpacing);
\draw[green!60!black, thick] (x1) to[out=0, in=180] (M1);
\draw[green!60!black, thick] (x2) to[out=0, in=180] (M1);
\draw[->, green!60!black, thick] (M1) to[out=0, in=180] (y1);

\coordinate (M2) at (\midX, -0.5*\nodeSpacing);
\draw[orange, thick] (x3) to[out=0, in=180] (M2);
\draw[orange, thick] (x4) to[out=0, in=180] (M2);
\draw[->, orange, thick] (M2) to[out=0, in=180] (y2);

\coordinate (M3) at (\midX, -0.8*\nodeSpacing); 
\draw[purple, thick] (x3) to[out=0, in=170] (M3);
\draw[purple, thick] (x4) to[out=0, in=190] (M3);
\draw[->, purple, thick] (M3) to[out=0, in=180] (y3);

\coordinate (M4) at (\midX, 0.5*\nodeSpacing);
\draw[cyan!80!black, thick] (x2) to[out=0, in=170] (M4);
\draw[cyan!80!black, thick] (x3) to[out=0, in=190] (M4);
\draw[->, cyan!80!black, thick] (M4) to[out=0, in=180] (y1);

\node[
    draw=black,          
    thick,               
    rounded corners=5pt, 
    inner sep=8pt,       
    text height=1.0cm,   
    text depth=1.0cm,    
    fill=black!10,
    minimum width=4cm    
] at (1.5, -4.5) {\footnotesize$\begin{aligned}&y_k = \sum _{i=0}^N w_{ki}x_i +\\&\sum _{0 \leq i \leq j}^{N-1} \tilde{w}_{kij}x_ix_j\end{aligned}$};

\draw[thick, rounded corners=15pt] 
    ([xshift=-0.22cm, yshift=-0.22cm]current bounding box.south west) 
    rectangle 
    ([xshift=0.22cm, yshift=0.22cm]current bounding box.north east);

\end{tikzpicture}
\end{center}

%% file: perceptron3.tex
\begin{center}
\begin{tikzpicture}[
    >=Stealth,
    neuron_style/.style={
        circle, draw=black, minimum size=1.0cm, inner sep=0pt,
        font=\sffamily\normalsize
    },
    input/.style={neuron_style, fill=blue!10},
    output/.style={neuron_style, fill=red!20},
    dotted_arrow/.style={->, thin, black!60, dashed}
]
\def\nodeSpacing{1.2}
\def\inX{0}
\def\outX{3}
\def\midX{1.5}

\foreach \i in {1,...,5}
    \node[input] (x\i) at (\inX, { (3-\i)*\nodeSpacing }) {$x_\i$};

\node[output] (y1) at (\outX, 1*\nodeSpacing) {$y_1$};
\node[output] (y2) at (\outX, 0*\nodeSpacing) {$y_2$};
\node[output] (y3) at (\outX, -1*\nodeSpacing) {$y_3$};

\draw[dotted_arrow] (x1) -- (y1);
\draw[dotted_arrow] (x5) -- (y3);
\draw[dotted_arrow] (x2) -- (y2);


\coordinate (C1_a) at (\midX, 1.6*\nodeSpacing);
\draw[green!60!black, thick] (x1) to[out=0, in=180] (C1_a);
\draw[green!60!black, thick] (x2) to[out=0, in=180] (C1_a);
\draw[->, green!60!black, thick] (C1_a) to[out=0, in=180] (y1);

\coordinate (C1_b) at (\midX, 1.1*\nodeSpacing);
\draw[green!60!black, thick] (x1) to[out=0, in=180] (C1_b);
\draw[green!60!black, thick] (x2) to[out=0, in=180] (C1_b);
\draw[->, green!60!black, thick] (C1_b) to[out=0, in=180] (y2);

\coordinate (C2_a) at (\midX, -0.2*\nodeSpacing);
\draw[red!70!black, thick] (x3) to[out=0, in=180] (C2_a);
\draw[red!70!black, thick] (x4) to[out=0, in=180] (C2_a);
\draw[->, red!70!black, thick] (C2_a) to[out=0, in=180] (y2);

\coordinate (C2_b) at (\midX, -0.8*\nodeSpacing);
\draw[red!70!black, thick] (x3) to[out=0, in=180] (C2_b);
\draw[red!70!black, thick] (x4) to[out=0, in=180] (C2_b);
\draw[->, red!70!black, thick] (C2_b) to[out=0, in=180] (y3);

\node[
    draw=black,          
    thick,               
    rounded corners=5pt, 
    inner sep=8pt,       
    text height=1.0cm,   
    text depth=1.0cm,    
    fill=black!10,
    minimum width=4cm    
] at (1.5, -4.5) {\footnotesize$\begin{aligned}&y_k = \sum _{i=0}^N w_{ki}x_i+\\&\sum _{i\leq j} (\tilde{\lambda} _{ij}^{(in)}-\tilde{\lambda} _k^{(out)})\tilde{\phi}_{kij}x_ix_j\end{aligned}$};

\draw[thick, rounded corners=15pt] 
    ([xshift=-0.22cm, yshift=-0.22cm]current bounding box.south west) 
    rectangle 
    ([xshift=0.22cm, yshift=0.22cm]current bounding box.north east);

\end{tikzpicture}
\end{center}

%% file: simple_expressivity_spectral_1D.tex
\begin{tikzpicture}[
    >=Stealth,
    neuron_style/.style={
        circle, draw=black, minimum size=1.4cm, inner sep=0pt,
        font=\sffamily\normalsize
    },
    input/.style={neuron_style, fill=blue!10},
    middle/.style={neuron_style},
    base/.style={neuron_style},
    lin/.style={->, dashed, black, line width=0.6pt},
    ho/.style={->, thick, orange},
    matnode/.style={draw=black, fill=white, rounded corners=2pt, inner sep=3pt, font=\sffamily\small},
    scale=0.8
]

\def\xIn{0}
\def\xY{3.9}
\def\xBase{7.9}
\def\nodeSpacing{3.0}

\node[input] (x0) at (\xIn,  {0.5*\nodeSpacing}) {$1$};
\node[input] (x1) at (\xIn, {-0.5*\nodeSpacing}) {$x$};

\node[middle] (y0) at (\xY,  { 1*\nodeSpacing}) {$y_0$};
\node[middle] (y1) at (\xY,  { 0*\nodeSpacing}) {$y_1$};
\node[middle] (y2) at (\xY,  {-1*\nodeSpacing}) {$y_2$};

\foreach \X [count=\i from 0] in {x0,x1}{
  \foreach \Y [count=\j from 0] in {y0,y1,y2}{
    \ifnum\i<\j\else
      \draw[lin] (\X) -- (\Y);
    \fi
  }
}

\def\hoshift{0.1}
\coordinate (t0) at ($(y0.west)+(-\hoshift,0)$);
\coordinate (t1) at ($(y1.west)+(-\hoshift,0)$);
\coordinate (t2) at ($(y2.west)+(-\hoshift,0)$);

\draw[thick, orange] (x1) to[out=18, in=180] (t0);
\draw[thick, orange] (x1) to[out=10, in=180] (t0);
\draw[->, thick, orange] (t0) -- (y0.west);

\draw[thick, orange] (x1) to[out=2,  in=180] (t1);
\draw[thick, orange] (x1) to[out=-2, in=180] (t1);
\draw[->, thick, orange] (t1) -- (y1.west);

\draw[thick, orange] (x1) to[out=-10, in=180] (t2);
\draw[thick, orange] (x1) to[out=-18, in=180] (t2);
\draw[->, thick, orange] (t2) -- (y2.west);

\node[base] (b0) at (\xBase, { 1*\nodeSpacing}) {$1$};
\node[base] (b1) at (\xBase, { 0*\nodeSpacing}) {$x$};
\node[base] (b2) at (\xBase, {-1*\nodeSpacing}) {$x^2$};

\foreach \Y in {y0,y1,y2}{
  \foreach \B in {b0,b1,b2}{
    \draw[lin] (\Y) -- (\B);
  }
}

\node[matnode, draw=none, font=\sffamily\large] (Winv) at ({(\xY+\xBase)/2}, 0) {$W^{-1}$};

\end{tikzpicture}

%% file: simple_expressivity_spectral_bias.tex
\begin{tikzpicture}[
    >=Stealth,
    neuron_style/.style={
        circle, draw=black, minimum size=1.4cm, inner sep=0pt,
        font=\sffamily\normalsize
    },
    input/.style={neuron_style, fill=blue!10},
    middle/.style={neuron_style},
    base/.style={neuron_style},
    lin/.style={->, dashed, black, line width=0.6pt},
    ho/.style={->, thick, orange},
    matnode/.style={draw=black, fill=white, rounded corners=2pt, inner sep=3pt, font=\sffamily\small},
    scale=0.8
]

\def\xIn{0}
\def\xY{4.3}
\def\xBase{8.7}
\def\nodeSpacing{3.2}

\node[input] (x0) at (\xIn,  {1*\nodeSpacing}) {$1$};
\node[input] (x1) at (\xIn,  {0*\nodeSpacing}) {$x_1$};
\node[input] (x2) at (\xIn, {-1*\nodeSpacing}) {$x_2$};

\node[middle] (y0) at (\xY,  {1.5*\nodeSpacing}) {$y_0$};
\node[middle] (y1) at (\xY,  {0.5*\nodeSpacing}) {$y_1$};
\node[middle] (y2) at (\xY, {-0.5*\nodeSpacing}) {$y_2$};
\node[middle] (y3) at (\xY, {-1.5*\nodeSpacing}) {$y_3$};

\draw[lin] (x0) -- (y0);

\draw[lin] (x1) -- (y0);
\draw[lin] (x1) -- (y1);

\draw[lin] (x2) -- (y0);
\draw[lin] (x2) -- (y1);
\draw[lin] (x2) -- (y2);

\def\hoshift{1.2}

\coordinate (MHO) at ({\xIn+\hoshift}, {-0.5*\nodeSpacing});

\draw[thick, orange] (x1) to[out=-15, in=180] (MHO);
\draw[thick, orange] (x2) to[out= 15, in=180] (MHO);

\draw[->, thick, orange] (MHO) to[out=0,  in=180] (y0.west);
\draw[->, thick, orange] (MHO) to[out=5,  in=180] (y1.west);
\draw[->, thick, orange] (MHO) to[out=-5, in=180] (y2.west);
\draw[->, thick, orange] (MHO) to[out=0,  in=180] (y3.west);

\node[base] (b0)  at (\xBase,  {1.5*\nodeSpacing}) {$1$};
\node[base] (b1)  at (\xBase,  {0.5*\nodeSpacing}) {$x_1$};
\node[base] (b2)  at (\xBase, {-0.5*\nodeSpacing}) {$x_2$};
\node[base] (b12) at (\xBase, {-1.5*\nodeSpacing}) {$x_1x_2$};

\foreach \Y in {y0,y1,y2,y3}{
  \foreach \B in {b0,b1,b2,b12}{
    \draw[lin] (\Y) -- (\B);
  }
}

\node[matnode, draw=none, font=\sffamily\Large] (Winv) at ({(\xY+\xBase)/2}, 0) {$W^{-1}$};

\def\xYnew{13.5}

\node[middle] (yy0) at (\xYnew, { 2*\nodeSpacing}) {$y_0$};
\node[middle] (yy1) at (\xYnew, { 1*\nodeSpacing}) {$y_1$};
\node[middle] (yy2) at (\xYnew, { 0*\nodeSpacing}) {$y_2$};
\node[middle] (yy3) at (\xYnew, {-1*\nodeSpacing}) {$y_3$};
\node[middle] (yy4) at (\xYnew, {-2*\nodeSpacing}) {$y_4$};

\foreach \B [count=\i from 0] in {b0,b1,b2,b12}{
  \foreach \Y [count=\j from 0] in {yy0,yy1,yy2,yy3}{
    \ifnum\i<\j\else
      \draw[lin] (\B) -- (\Y);
    \fi
  }
}

\def\hoshiftB{0.9}

\coordinate (tt0) at ($(yy0.west)+(-\hoshiftB,0)$);
\coordinate (tt1) at ($(yy1.west)+(-\hoshiftB,0)$);
\coordinate (tt2) at ($(yy2.west)+(-\hoshiftB,0)$);
\coordinate (tt3) at ($(yy3.west)+(-\hoshiftB,0)$);
\coordinate (tt4) at ($(yy4.west)+(-\hoshiftB,0)$);

\draw[thick, orange] (b1) to[out=18, in=180] (tt0);
\draw[thick, orange] (b1) to[out=12, in=180] (tt0);
\draw[->, thick, orange] (tt0) -- (yy0.west);

\draw[thick, orange] (b1) to[out=10, in=180] (tt1);
\draw[thick, orange] (b1) to[out=6,  in=180] (tt1);
\draw[->, thick, orange] (tt1) -- (yy1.west);

\draw[thick, orange] (b1) to[out=2,  in=180] (tt2);
\draw[thick, orange] (b1) to[out=-2, in=180] (tt2);
\draw[->, thick, orange] (tt2) -- (yy2.west);

\draw[thick, orange] (b1) to[out=-6,  in=180] (tt3);
\draw[thick, orange] (b1) to[out=-10, in=180] (tt3);
\draw[->, thick, orange] (tt3) -- (yy3.west);

\draw[thick, orange] (b1) to[out=-12, in=180] (tt4);
\draw[thick, orange] (b1) to[out=-18, in=180] (tt4);
\draw[->, thick, orange] (tt4) -- (yy4.west);

\node[font=\sffamily\huge] at ({\xYnew+2.2}, 0) {$\cdots$};

\end{tikzpicture}

%% file: simple_expressivity.tex
\begin{tikzpicture}[
    >=Stealth,
    neuron_style/.style={
        circle, draw=black, minimum size=1.0cm, inner sep=0pt,
        font=\sffamily\normalsize
    },
    input/.style={
        neuron_style, fill=blue!10
    },
    output/.style={
        neuron_style, fill=red!20
    },
    middle/.style={
        neuron_style
    },
    arrow/.style={
        ->, thick, black
    },
    dotted_arrow/.style={
        ->, thin, black!60, dashed
    },
    label/.style={
        font=\sffamily\small
    }
]

\def\inputXL{0}
\def\outputXL{4.0}
\def\inputXR{8.0} 
\def\outputXR{13.0} 
\def\nodeSpacing{2} 


\node[input] (x1) at (\inputXL, 0) {$x$};

\node[middle] (y1) at (\outputXL, 1*\nodeSpacing) {$1$};
\node[middle] (y2) at (\outputXL, 0*\nodeSpacing) {$x$};
\node[middle] (y3) at (\outputXL, -1*\nodeSpacing) {$x^2$};

\node[middle] (ys1) at (2*\outputXL, 2*\nodeSpacing) {$1$};
\node[middle] (ys2) at (2*\outputXL, 1*\nodeSpacing) {$x$};
\node[middle] (ys3) at (2*\outputXL, 0*\nodeSpacing) {$x^2$};
\node[middle] (ys4) at (2*\outputXL, -1*\nodeSpacing) {$x^3$};
\node[middle] (ys5) at (2*\outputXL, -2*\nodeSpacing) {$x^4$};

\node[output] (o1) at (3*\outputXL, 0.5*\nodeSpacing) {$p^{(4+)}$};
\node[output] (o2) at (3*\outputXL, -0.5*\nodeSpacing) {$p^{(4+)}$};

\draw[arrow, dashed] (x1) -- (y1) node[midway, above, sloped, label] {$w_{0,1},\tilde{w}_{0,1,1}=0$};
\draw[arrow] (x1) -- (y2) node[midway, above, sloped, label] {$w_{1,0},\tilde{w}_{1,1,1}=0$};

\coordinate (M1) at (2, -0.6*\nodeSpacing);

\draw[thick] (x1) to[out=-25, in=180] (M1);
\draw[thick] (x1) to[out=-50, in=180] (M1);
\draw[->, thick] (M1) to[out=0, in=180] (y3);

\node[black, font=\footnotesize, below left] at (M1) {$w_{2,0},w_{2,1}=0$};

\draw[arrow, dashed] (y1) -- (ys1) node[midway, above, sloped, label] {};

\draw[arrow] (y2) -- (ys2) node[midway, above, sloped, label] {};

\draw[arrow] (y3) -- (ys3) node[midway, above, sloped, label] {};

\coordinate (M2) at (6, -0.8*\nodeSpacing);

\draw[thick] (y2) to[out=0, in=180] (M2);
\draw[thick] (y3) to[out=0, in=180] (M2);
\draw[->, thick] (M2) to[out=0, in=180] (ys4);

\coordinate (M3) at (6, -1.8*\nodeSpacing);

\draw[thick] (y3) to[out=-25, in=180] (M3);
\draw[thick] (y3) to[out=-50, in=180] (M3);
\draw[->, thick] (M3) to[out=0, in=180] (ys5);


\coordinate (M4) at (10, -0.8*\nodeSpacing);

\draw[arrow, gray] (ys1) -- (o1) node[midway, above, sloped, label] {};
\draw[arrow, gray] (ys1) -- (o2) node[midway, above, sloped, label] {};
\draw[arrow, gray] (ys4) -- (o1) node[midway, above, sloped, label] {};

\draw[thick, gray] (ys3) to[out=0, in=180] (M4);
\draw[thick, gray] (ys5) to[out=0, in=180] (M4);
\draw[->, thick, gray] (M4) to[out=0, in=180] (o2);

\end{tikzpicture}